\documentclass{article}

\PassOptionsToPackage{numbers}{natbib}

\usepackage[preprint]{nips_2024}

\usepackage{makeidx}

\usepackage{graphicx}
\usepackage{hyperref}
\usepackage{xcolor}         
\usepackage{subcaption}

\usepackage[utf8]{inputenc} 
\usepackage[T1]{fontenc}    
\usepackage{wrapfig}
\usepackage{url}
\usepackage{makecell}
\usepackage{textcomp}
\usepackage{booktabs}
\usepackage{multirow}
\usepackage[normalem]{ulem}
\useunder{\uline}{\ul}{}

\usepackage{amsfonts}       
\usepackage{nicefrac}       
\usepackage{microtype}      
\usepackage[toc,page]{appendix}
\usepackage[font=small,labelfont=bf]{caption}

\usepackage{algorithm}
\usepackage{algorithmic}

\usepackage[cmex10]{amsmath}
\usepackage{amssymb,amsthm}

\usepackage{amsmath,amsfonts,bm}

\def\eqref#1{equation~\ref{#1}}
\def\Eqref#1{Eq.~\ref{#1}}

\def\1{\bm{1}}

\def\rva{{\mathbf{a}}}
\def\rvb{{\mathbf{b}}}

\def\rvo{{\mathbf{o}}}

\def\rvs{{\mathbf{s}}}

\def\mW{{\bm{W}}}

\DeclareMathAlphabet{\mathsfit}{\encodingdefault}{\sfdefault}{m}{sl}
\SetMathAlphabet{\mathsfit}{bold}{\encodingdefault}{\sfdefault}{bx}{n}

\def\gA{{\mathcal{A}}}

\def\gE{{\mathcal{E}}}

\def\gH{{\mathcal{H}}}

\def\gL{{\mathcal{L}}}

\def\gN{{\mathcal{N}}}
\def\gO{{\mathcal{O}}}

\def\gS{{\mathcal{S}}}

\def\sA{{\mathbb{A}}}

\def\sI{{\mathbb{I}}}

\def\sO{{\mathbb{O}}}

\def\sR{{\mathbb{R}}}
\def\sS{{\mathbb{S}}}

\newcommand{\E}{\mathbb{E}}

\newcommand{\R}{\mathbb{R}}

\newcommand{\KL}{D_{\mathrm{KL}}}

\newcommand{\x}{\times}

\newcommand{\cA}{\mathcal{A}}

\newcommand{\cO}{\mathcal{O}}

\DeclareMathOperator*{\argmax}{arg\,max}

\renewcommand{\citename}{\citet}
\renewcommand{\cite}{\citep}

\definecolor{myLinkColor}{rgb}{0.18,0.39,0.62}
\hypersetup{
  colorlinks=true,
  linkcolor=myLinkColor,
  filecolor=myLinkColor,
  urlcolor=myLinkColor,
  citecolor=myLinkColor,
}

\newtheorem{thm}{Theorem}

\newcommand{\appref}[1]{\hyperref[#1]{Appendix~\ref*{#1}}}

\newcommand{\BlackBox}{\rule{1.5ex}{1.5ex}}  
\ifdefined\proof
    \renewenvironment{proof}{\par\noindent{\bf Proof\ }}{\hfill\BlackBox\\[2mm]}
\else
    \newenvironment{proof}{\par\noindent{\bf Proof\ }}{\hfill\BlackBox\\[2mm]}
\fi

\newtheorem{theorem}{Theorem}

\newtheorem{proposition}[theorem]{Proposition}
\newtheorem{remark}[theorem]{Remark}

\newtheorem{definition}[theorem]{Definition}

\newcommand\Dist{\Delta}

\newtheorem{claim}[theorem]{Claim}
\usepackage{tikz-cd}
\usepackage{tikz}

\title{Learning Temporal Abstractions via Variational
  Homomorphisms in Option-Induced Abstract MDPs}

\author{%
  Chang Li$^\textbf{1}$, Yaren Zhang$^\textbf{2}$, Haoran Lv$^\textbf{3}$\thanks{The work does not relate to the author's position at Amazon.}, Qiong Cao$^\textbf{1}$, Chao Xue$^\textbf{1}$, Xiaodong He$^\textbf{1}$ \\
  $^\textbf{1}$JD Joy Future Academy, China \\ \texttt{\{lichang93, caoqiong1, xuechao19, xiaodong.he\}@jd.com}\\
  $^\textbf{2}$ Carleton University, Canada \\ \texttt{\{yarenzhang\}@cmail.carleton.ca} \\
  $^\textbf{3}$Amazon Web Services, China \\ \texttt{\{lvhaoran\}@amazon.com} \\
}

\begin{document}

\maketitle

\begin{abstract}
  Large Language Models (LLMs) have shown remarkable reasoning
  ability through explicit Chain-of-Thought (CoT) prompting, but
  generating these step-by-step textual explanations is
  computationally expensive and slow. To overcome this, we aim to
  develop a framework for efficient, implicit reasoning, where
  the model ``thinks'' in a latent space without generating
  explicit text for every step. We propose that these latent
  thoughts can be modeled as temporally-extended abstract
  actions, or ``options,'' within a hierarchical reinforcement
  learning framework. To effectively learn a diverse library of
  options as latent embeddings, we first introduce the
  Variational Markovian Option Critic (VMOC), an off-policy
  algorithm that uses variational inference within the HiT-MDP
  framework. To then provide a rigorous foundation for using
  these options as an abstract reasoning space, we extend the
  theory of continuous MDP homomorphisms. This proves that
  learning a policy in the simplified, abstract latent space—for
  which VMOC is suited—preserves the optimality of the solution
  to the original, complex problem. Finally, we propose a
  cold-start procedure that leverages supervised fine-tuning
  (SFT) data to distill human reasoning demonstrations into this
  latent option space, providing a rich initialization for the
  model's reasoning capabilities. Extensive experiments
  demonstrate that our approach achieves strong performance on
  complex logical reasoning benchmarks and challenging locomotion
  tasks, validating our framework as a principled method for
  learning abstract skills for both language and control.
\end{abstract}

\section{Introduction}

Recent advancements in deep reinforcement learning (DRL) have
demonstrated significant successes across a variety of complex
domains, such as mastering the human level of
atari~\cite{mnih2015human} and Go~\cite{silver2016mastering}
games. These achievements underscore the potential of combining
reinforcement learning (RL) with powerful function approximators
like neural networks~\cite{bertsekas1996neuro} to tackle
intricate tasks that require nuanced control over extended
periods. Despite these breakthroughs, Deep RL still faces
substantial challenges, such as insufficient exploration in
dynamic
environments~\cite{haarnoja2017reinforcement,eysenbach2018diversity,sharma2019dynamics},
inefficient learning associated with temporally extended
actions~\cite{brockett1993hybrid,colombetti1996behavior} and long
horizon tasks~\cite{konidaris2009skill,bacon2017option}, and vast
amounts of samples required for training proficient
behaviors~\cite{guo2017using,schulman2017equivalence,goyal2019infobot}.

One promising area for addressing these challenges is the
utilization of hierarchical reinforcement learning
(HRL)~\cite{dayan1993feudal,araujo1996learning,dietterich2000hierarchical},
a diverse set of strategies that decompose complex tasks into
simpler, hierarchical structures for more manageable learning.
Among these strategies, the option
framework~\cite{sutton1999between}, developed on the Semi-Markov
Decision Process (SMDP), is particularly effective at segmenting
non-stationary task stages into temporally-extended actions known
as options. Options are typically learned through a maximum
likelihood approach that aims to maximize the expected rewards
across trajectories. In this framework, options act as temporally
abstracted actions executed over variable time steps, controlled
by a master policy that decides when each option should execute
and terminate. This structuring not only simplifies the
management of complex environments but also enables the
systematic discovery and execution of temporal abstractions over
long-horizon tasks~\cite{khetarpal2019learning,
  kamat2020diversity}.

However, the underlying SMDP framework is frequently undermined
by three key challenges: 1) Insufficient exploration and
degradation~\cite{harb2018waiting, osa2019hierarchical,
  kamat2020diversity}. As options are unevenly updated using
conventional maximum likelihood methods~\cite{bacon2017option,
  daniel2016probabilistic, smith2018inference,
  khetarpal2020options, klissarov2021flexible}, the policy is
quickly saturated with early rewarding observations. This
typically results in focusing on only low-entropy options that
lead to local optima rewards, causing a single option to either
dominate the entire policy or switch every timestep. Such
premature convergence limits option diversity significantly. 2)
Sample Inefficiency. The semi-Markovian nature inherently leads
to sample
inefficiency~\cite{sutton1999between,kolobov2012discovering}:
each policy update at the master level extends over multiple time
steps, thus consuming a considerable volume of experience samples
with relatively low informational gain. This inefficiency is
further exacerbated by the prevalence of on-policy option
learning algorithms~\cite{ bacon2017option, zhang2019dac}, which
require new samples to be collected simultaneously from both
high-level master policies and low-level action policies at each
gradient step, and thus sample expensive. 3) Computationally
expensive. Options are conventionally defined as
triples~\cite{bacon2017option} with intra-option policies and
termination functions, often modeled using neural networks which
are expensive to optimize. These challenges collectively limit
the broader adoption and effectiveness of the option framework in
real-world scenarios, particularly in complex continuous
environments where scalability and stability are
critical~\cite{fujimoto2018addressing,li2020skill,klissarov2021flexible}.

To address these challenges, we introduce the Variational
Markovian Option Critic (VMOC), a novel off-policy algorithm that
integrates the variational inference framework on option-induced
MDPs~\cite{li2022hit}. We first formulate the optimal
option-induced SMDP trajectory as a probabilistic inference
problem, presenting a theoretical convergence proof of the
variational distribution under the soft policy iteration
framework~\cite{haarnoja2018soft}. Similar to prior variational
methods~\cite{levine2018reinforcement}, policy entropy terms
naturally arise as intrinsic rewards during the inference
procedure. As a result, VMOC not only seeks high-reward options
but also maximizes entropy across the space, promoting extensive
exploration and maintaining high diversity. We implements this
inference procedure as an off-policy soft actor
critic~\cite{haarnoja2018soft} algorithm, which allows reusing
samples from replay buffer and enhances sample efficiency.
Furthermore, to address the computational inefficiencies
associated with conventional option triples, we
follow~\cite{li2022hit} and employ low-cost option embeddings
rather than complex neural network models. This not only
simplifies the training process but also enhances the
expressiveness of the model by allowing the agent to capture a
more diverse set of environmental dynamics.

To provide a rigorous theoretical foundation for learning in
abstract option spaces, we extend the theory of continuous MDP
homomorphisms~\cite{panangaden2024policy} to the continuous
HiT-MDP setting. MDP homomorphisms provide a formal framework for
state-action abstractions that preserve optimal value functions,
but previous work has been limited to standard continuous MDPs.
We introduce continuous HiT-MDP homomorphisms using the
mathematical framework of vector bundles to elegantly capture the
relationship between state-option pairs across different levels
of abstraction. This formulation allows us to prove that optimal
value equivalence and policy lifting properties extend to the
option framework, ensuring that learning in abstract spaces does
not sacrifice optimality.

Building on this theoretical foundation, we further establish
that the variational inference framework seamlessly integrates
with abstract HiT-MDPs. Specifically, we prove that maximizing
the evidence lower bound (ELBO) in an abstract HiT-MDP obtained
through a homomorphism is equivalent to optimizing a lower bound
of the ELBO in the original space. This result provides a
principled justification for using VMOC to learn policies in
abstract option spaces: the algorithm directly optimizes the
variational objective of the original problem while benefiting
from the computational advantages of working in a simplified
abstract space. The combination of HiT-MDP homomorphisms and
variational inference thus offers both theoretical guarantees and
practical benefits for hierarchical reinforcement learning.

Beyond traditional control tasks, the structure of VMOC offers a
compelling solution to challenges in other domains, such as the
efficiency of reasoning in Large Language Models (LLMs). While
Chain-of-Thought (CoT) prompting has enabled LLMs to tackle
complex multi-step problems, the generation of explicit reasoning
text incurs substantial computational and latency
costs~\cite{wei2022chain}. A promising alternative is to perform
reasoning implicitly within the model's latent
space~\cite{geiping2025scaling, ruan2025reasoning}, but this
often sacrifices the interpretability of the reasoning
process~\cite{geiping2025scaling}. Our framework naturally
addresses this trade-off. We posit that the latent option space
in VMOC can represent abstract reasoning steps, or an ``implicit
CoT''. To initialize this space, we propose a cold-start phase
that leverages supervised fine-tuning (SFT) datasets of explicit
reasoning demonstrations. Through a variational objective
analogous to \Eqref{eq:elbo}, we distill these explicit CoT
chains into the discrete latent option embeddings. This
pre-training endows the model with a rich library of reasoning
primitives that can be invoked for efficient, purely latent
inference. This two-stage approach opens the door to developing
agents that perform implicit reasoning in a structured latent
space, while retaining the ability to be refined via RL and
potentially be decoded into understandable language, bridging
efficient performance with verifiability.

Our contributions can be summarized as follows:
\begin{itemize}
\item We propose the Variational Markovian Option Critic (VMOC),
  an off-policy, maximum-entropy algorithm that learns a diverse
  set of options represented as low-cost embeddings, enhancing
  sample efficiency and exploration.
\item We provide a rigorous theoretical foundation by extending
  the framework of continuous MDP homomorphisms to HiT-MDPs,
  proving that learning in the abstract option space preserves
  optimality guarantees.
\item We introduce a novel application for VMOC in language
  models, proposing a cold-start supervised fine-tuning (SFT)
  procedure to learn an ``implicit Chain-of-Thought'' in the
  latent option space for efficient and effective reasoning.
\item We demonstrate through extensive experiments that VMOC
  significantly outperforms strong baselines in challenging
  Mujoco locomotion tasks and achieves competitive results on
  complex logical reasoning benchmarks, validating its broad
  applicability.
\end{itemize}

\section{Preliminary}
\label{sec:back}

\subsection{Control as Structured Variational Inference}
\label{sec:back_var}

Conventionally, the control as inference
framework~\cite{haarnoja2018soft, levine2018reinforcement,
  haarnoja2018soft,ziebart2010modeling} is derived using the
maximum entropy objective. In this section, we present an
alternative derivation from the perspective of structured
variational inference. We demonstrate that this approach provides
a more concise and intuitive pathway to the same theoretical
results, where the maximum entropy principle naturally emerges
through the direct application of variational inference
techniques.

Traditional control methods focus on directly maximizing rewards,
often resulting in suboptimal trade-offs between exploration and
exploitation. By reinterpreting the control problem as a
probabilistic inference problem, the control as inference
framework incorporates both the reward structure and
environmental uncertainty into decision-making, providing a more
robust and flexible approach to policy optimization. In this
framework, optimality is represented by a binary random variable
\(\gE\in \{0,1\}\)\footnote{Conventionally, the optimality
  variable is denoted by \(\mathcal{O}\). However, in this
  context, we use \(\gE\) to avoid conflict with notation used in
  the option framework.}. The probability of optimality given a
state-action pair \((\rvs, \rva)\) is denoted as \(P(\gE=1\mid
\rvs,\rva)=\exp(r(\rvs,\rva))\), which is an exponential function
of the conventional reward function \(r(\rvs,\rva)\) that
measures the desirability of an action in a specific state.
Focusing on \(\gE=1\) captures the occurrence of optimal events.
For simplicity, we will use \(\gE\) instead of \(\gE=1\) in the
following text to avoid cluttered notations. The joint
distribution over trajectories \(\tau = (\rvs_1, \rva_1, \ldots,
\rvs_T, \rva_T)\) given optimality is expressed as:
\[
  P(\tau|\gE_{1:T}) \propto P(\tau,\gE_{1:T}) = P(\rvs_1)
  \prod_{t=1}^{T-1} P(\rvs_{t+1}|\rvs_t,\rva_t)
  P(\gE_t|\rvs_t,\rva_t)
\]
where \(P(\rvs_1)\) is the initial state distribution,
\(P(\rvs_{t+1}|\rvs_t,\rva_t)\) is the dynamics model. As explained in~\cite{haarnoja2018soft,levine2018reinforcement},
direct optimization of \(P(\tau\mid \gE_{1:T})\) can result in an
optimistic policy that assumes a degree of control over the
dynamics. One way to correct this risk-seeking
behavior~\cite{levine2018reinforcement} is through structured
variational inference. In our case, the goal is to approximate
the optimal trajectory \(P(\tau)\) with the variational
distribution:
\[
q(\tau) = P(\rvs_1) \prod_{t=1}^{T-1} P(\rvs_{t+1}\mid \rvs_t,\rva_t) q(\rva_t\mid \rvs_t)
\]
where the initial distribution \(P(\rvs_1)\) and transition
distribution \(P(\rvs_{t+1}\mid \rvs_t,\rva_t)\) is set to be the
true environment dynamics from \(P(\tau)\). The only variational
term is the variational policy \(q(\rva_t\mid \rvs_t)\), which is
used to approximate the optimal policy \(P(\rva_t\mid
\rvs_t,\gE_{1:T})\). Under this setting, the environment dynamics
will be canceled out from the optimization objective between
\(P(\tau\mid \gE)\) and \(q(\tau)\), thus explicitly disallowing
the agent to influence its dynamics and correcting the
risk-seeking behavior.

With the variational distribution at hand, the conventional
maximum entropy framework can be recovered through a direct
application of standard structural variational
inference~\cite{koller2009probabilistic}:
\begin{align*}
  \log P(\gE_{1:T}) &= \gL(q(\tau), P(\tau,\gE_{1:T})) + \KL(q(\tau)\parallel P(\tau|\gE_{1:T}))\\
                    &= \underbrace{\E_{\tau \sim q(\tau)}[\sum_t^Tr(\rvs_t,\rva_t) + \gH( q(\cdot|\rvs_t) ) ]}_\text{maximum entropy objective} + \KL(q(\rva_t|\rvs_t)\parallel P(\rva_t|\rvs_t,\gE_{1:T}))
\end{align*}
where $\gL(q,P)=\E_{q}[\log \frac{P}{q}]$ is the Evidence Lower
Bound (ELBO)~\cite{koller2009probabilistic}. The maximum entropy
objective arises naturally as the environment dynamics in
$P(\tau,\gE)$ and $q(\tau)$ cancel out. Under this formulation,
the soft policy iteration theorem~\cite{haarnoja2018soft} has an
elegant Expectation-Maximization (EM)
algorithm~\cite{koller2009probabilistic} interpretation: the
E-step corresponds to the policy evaluation of the maximum
entropy objective $\mathcal{L}(q^{[k]}, P)$; while the M-step
corresponds to the policy improvement of the $\KL$ term
$q^{[k+1]}=\argmax_q\KL(q^{[k]}(\tau)\parallel P(\tau\mid\gE))$.
Thus, soft policy iteration is an exact inference if both EM
steps can be performed exactly.
\begin{thm}[Convergence Theorem for Soft Policy Iteration]
  \label{thm:var_soft_mdp}
  Let \(\tau\) be the latent variable and \(\gE\) be the observed
  variable. Define the variational distribution \(q(\tau)\) and
  the log-likelihood \(\log P(\gE)\). Let \( M: q^{[k]}
  \rightarrow q^{[k+1]} \) represent the mapping defined by the
  EM steps inference update, so that \( q^{[k+1]} = M(q^{[k]})
  \). The likelihood function increases at each iteration of the
  variational inference algorithm until convergence conditions
  are satisfied.
\begin{proof}
  See \appref{app:thm_var_soft_mdp}.
\end{proof}
\end{thm}

\subsection{The Option Framework}
\label{sec:back_option}
In conventional SMDP-based Option
Framework~\cite{sutton1999between}, an option is a triple
$(\sI_o, \pi_o, \beta_o)\in \gO$, where $\gO$ denotes the option
set; $o\in\sO=\{1,2,\dots,K\}$ is a positive integer index which
denotes the $o$-th triple where $K$ is the number of options;
$\sI_o$ is an initiation set indicating where the option can be
initiated; $\pi_o=P_o(\rva|\rvs):\sA\times\sS\rightarrow[0,1]$ is
the action policy of the $o$th option;
$\beta_o=P_o(\rvb=1|\rvs):\sS\rightarrow[0,1]$ where $\rvb\in
\{0,1\}$ is a \emph{termination function}. For clarity, we use
$P_o(\rvb=1|\rvs)$ instead of $\beta_o$ which is widely used in
previous option literatures (e.g.,
\citet{sutton1999between,bacon2017option}). A \emph{master
  policy} $\pi(\rvo|\rvs)=P(\rvo|\rvs)$ where $\rvo\in\sO$ is
used to sample which option will be executed. Therefore, the
dynamics (stochastic process) of the option framework is written
as:
\begin{align}
  P(\tau) = P(\rvs_0,\rvo_0)&\prod_{t=1}^\infty P(\rvs_t|\rvs_{t-1},\rva_{t-1})P_{o_t}(\rva_{t}|\rvs_{t})\nonumber\\
                     &[P_{o_{t-1}}(\rvb_t=0|\rvs_t)\1_{\rvo_t=o_{t-1}} +P_{o_{t-1}}(\rvb_t=1|\rvs_t)P(\rvo_t|\rvs_t)],
                       \label{eq:smdp_pgm}
\end{align}
where $\tau=\{\rvs_0,\rvo_0,\rva_0,\rvs_1,\rvo_1,\rva_1,\ldots\}$
denotes the trajectory of the option framework. $\1$ is an
indicator function and is only true when $\rvo_t=o_{t-1}$ (notice
that $o_{t-1}$ is the realization at $\rvo_{t-1}$). Therefore,
under this formulation the option framework is defined as a
Semi-Markov process since the dependency on an activated option
$o$ can cross a variable amount of time \cite{sutton1999between}.
Due to the nature of SMDP assumption, conventional option
framework is unstable and computationally expensive to optimize.
Li et al.~\cite{li2020skill,li2022hit} proposed the Hidden
Temporal Markovian Decision Process (HiT-MDP):
\begin{equation}
  \label{eq:hit_mdp_pgm}
  P(\tau)=P(\rvs_0,\rvo_0)
  \prod_{t=1}^\infty P(\rvs_t|\rvs_{t-1},\rva_{t-1})
  P(\rva_{t}|\rvs_{t},\rvo_{t})P(\rvo_t|\rvs_t,\rvo_{t-1})
\end{equation}
and theoretically proved that the option-induced HiT-MDP is
homomorphically equivalent to the conventional SMDP-based option
framework. Following RL conventions, we use
$\pi^A=P(\rva_{t}|\rvs_{t},\rvo_{t})$ to denote the action policy
and $\pi^O=P(\rvo_t|\rvs_t,\rvo_{t-1})$ to denote the option
policy respectively. In HiT-MDPs, options can be viewed as latent
variables with a temporal structure
$P(\rvo_t|\rvs_t,\rvo_{t-1})$, enabling options to be represented
as dense latent embeddings rather than traditional option
triples. They demonstrated that learning options as embeddings on
HiT-MDPs offers significant advantages in performance,
scalability, and stability by reducing variance. However, their
work only derived an on-policy policy gradient algorithm for
learning options on HiT-MDPs. In this work, we extend their
approach to an off-policy algorithm under the variational
inference framework, enhancing exploration and sample efficiency.

\section{A Variational Approach for Learning HiT-MDP}
\label{sec:meth}

In this section, we introduce the Variational Markovian Option
Critic (VMOC) algorithm by extending the variational policy
iteration (Theorem~\ref{thm:var_soft_mdp}) to the option
framework. In Section~\ref{sec:meth_reform}, we reformulate the
optimal option trajectory and the variational distribution as
probabilistic graphical models (PGMs), propose the corresponding
variational objective, and present a provable exact inference
procedure for these objectives in tabular settings.
Section~\ref{sec:meth_vmoc} extends this result by introducing
VMOC, a practical off-policy option learning algorithm that uses
neural networks as function approximators and proves the
convergence of VMOC under approximate inference settings. Our
approach differs from previous
works~\cite{haarnoja2018soft,li2020soac,li2020skill} by
leveraging structured variational inference directly, providing a
more concise pathway to both theoretical results and practical
algorithms.

\subsection{PGM Formulations of The Option Framework}
\label{sec:meth_reform}

Formulating complex problems as probabilistic graphical models
(PGMs) offers a consistent and flexible framework for deriving
principled objectives, analyzing convergence, and devising
practical algorithms. In this section, we first formulate the
optimal trajectory of the conventional SMDP-based option
framework (\Eqref{eq:smdp_pgm}) as a PGM. We then use the
HiT-MDPs as the variational distribution to approximate this
optimal trajectory. With these PGMs, we can straightforwardly
derive the variational objective, where maximum entropy terms
arise naturally. This approach allows us to develop a stable
algorithm for learning diversified options and preventing
degeneracy.
\begin{figure*}[thb]
  \centering
  \includegraphics[width=1\linewidth]{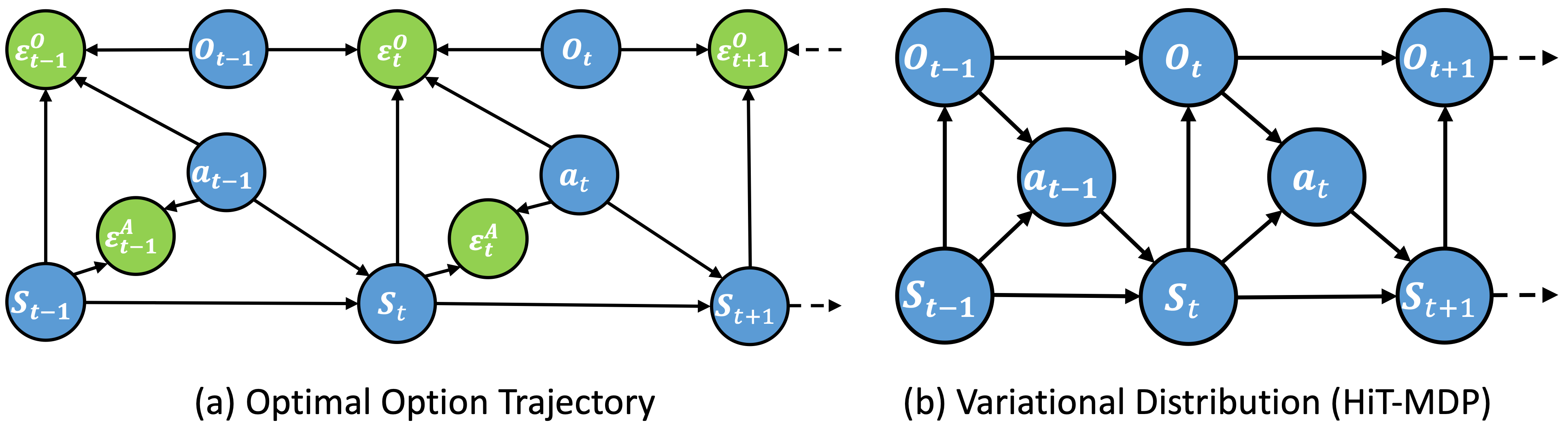}\\
  \caption{\label{fig:var_pgm} PGMs of the option framework.}
\end{figure*}
Specifically, we follow \cite{levine2018reinforcement,
  koller2009probabilistic} by introducing the concept of
"Optimality" \cite{todorov2006linearly} into the conventional
SMDP-based option framework (Equation~\eqref{eq:smdp_pgm}). This
allows us to define the probability of an option trajectory being
optimal as a probabilistic graphical model (PGM), as illustrated
in Figure~\ref{fig:var_pgm} (a):
\begin{align}
    \label{eq:maxent_pgm}
  \nonumber   P(\tau,\gE^A_{1:T},\gE^O_{1:T}&)=P({\rvs}_{0},\rvo_{0})\prod_{t=1}^{T} P({\rvs}_{t+1}|\rvs_{t},\rva_{t}) P(\gE^A_t=1|\rvs_t,\rva_t) P(\gE^O_t=1|\rvs_t,\rva_t,\rvo_t,\rvo_{t-1})P(\rvo_t)P(\rva_t)\\
                                            & \propto \underbrace{P({\rvs}_{0}) \prod_{t=1}^{T}P({\rvs}_{t+1} | {\rvs}_{t}, {\rva}_{t})}_\text{Environment Dynamics}\underbrace{\prod_{t=1}^{T} P(\gE^A_t=1|\rvs_t,\rva_t) P(\gE^O_t=1|\rvs_t,\rva_t,\rvo_t,\rvo_{t-1})}_\text{Optimality Likelihood},
\end{align}
where $\gE\in\{0,1\}$ are observable binary ``optimal random
variables''~\cite{levine2018reinforcement},
$\tau=\{\rvs_0,\rvo_0,\rva_0,\rvs_1\ldots\}$ denotes the trajectory of
the option framework. The agent is \emph{optimal} at time step
$t$ when $P(\gE^A_t=1|\rvs_t,\rva_t)$ and
$P(\gE^O_t=1|\rvs_t,\rva_t,\rvo_t,\rvo_{t-1})$. We will use
\(\gE\) instead of \(\gE=1\) in the following text to avoid
cluttered notations. To simplify the derivation, priors $P(\rvo)$
and $P(\rva)$ can be assumed to be uniform distributions without
loss of generality \cite{levine2018reinforcement}. Note that
Eq.~\ref{eq:maxent_pgm} shares the same environment dynamics with
\Eqref{eq:smdp_pgm} and \Eqref{eq:hit_mdp_pgm}. With the optimal
random variables $\gE^O$ and $\gE^A$, the likelihood of a
state-action $\{\rvs_t,\rva_t\}$ pair that is optimal is defined
as:
\begin{equation}
  P(\gE^A_{t} | \rvs_{t}, \rva_{t})=\exp (r(\rvs_{t}, \rva_{t})),
\label{eq:e_ras}
\end{equation}
as this specific design facilitates recovering the value function
at the latter structural variational inference stage. Based on
the same motivation, the likelihood of an option-state-action
$\{\rvo_t,\rvs_t,\rva_t,\rvo_{t-1}\}$ pair that is optimal is
defined as,
\begin{equation}
  P(\gE^O_{t} |\rvs_{t}, \rva_{t},\rvo_t,\rvo_{t-1})= \exp(f(\rvo_t,\rvs_t,\rva_t,\rvo_{t-1})),
  \label{eq:e_ros}
\end{equation}
where $f(\cdot)$ is an arbitrary non-positive function which
measures the preferable of selecting an option given state-action
pair $[\rvs_t, \rva_t]$ and the previous executed option
$\rvo_{t-1}$. In this work, we choose $f$ to be the
mutual-information $f=I[\rvo_t|\rvs_t,\rva_t,\rvo_{t-1}]$ as a
fact that when the uniform prior assumption of $P(\rvo)$ is
relaxed the optimization introduces a mutual-information as a
regularizer~\cite{li2022hit}.

As explained in Section~\ref{sec:back_var}, direct optimization
of \Eqref{eq:maxent_pgm} results in optimistic policies that
assumes a degree of control over the dynamics. We correct this
risk-seeking behavior~\cite{levine2018reinforcement} through
approximating the optimal trajectory \(P(\tau)\) with the
variational distribution:
\begin{equation}
  \label{eq:var_hit_mdp}
  q(\tau) = P(\rvs_0,\rvo_0) \prod_{t=1}^{T-1} P(\rvs_{t+1}|
  \rvs_t,\rva_t) q(\rva_t|\rvs_t, \rvo_t) q(\rvo_t| \rvs_t,
  \rvo_{t-1})
\end{equation}
where the initial distribution \(P(\rvs_0,\rvo_0)\) and
transition distribution \(P(\rvs_{t+1}\mid \rvs_t,\rva_t)\) is
set to be the true environment dynamics from \(P(\tau)\). The
variational distribution turns out to be the HiT-MDP, where the
action policy \(q(\rva_t\mid \rvs_t)\) and the option policy
\(q(\rvo_t| \rvs_t, \rvo_{t-1})\) are used to approximate the
optimal policy \(P(\rva_t| \rvs_t,\rvo_t,\gE^A_{1:T})\) and
\(P(\rvo_t| \rvs_t, \rvo_{t-1},\gE^O_{1:T})\). The Evidence Lower
Bound (ELBO)~\cite{koller2009probabilistic} of the log-likelihood
optimal trajectory (\Eqref{eq:maxent_pgm}) can be derived as (see
\appref{app:deriv_lme_hitmdp}):
\begin{align}
  \label{eq:elbo}
  \nonumber \gL(q(\tau),P(\tau,\gE^A_{1:T},\gE^O_{1:T}))&=\E_{q(\tau)}[\log P(\tau,\gE^A_{1:T},\gE^O_{1:T})-\log q(\tau)]\\
  \nonumber &=\E_{q(\tau)}[r(\rvs_t,\rva_t)+f( \cdot )-\log q(\rva_t|\rvs_t,\rvo_t) - \log q(\rvo_t| \rvs_t, \rvo_{t-1})]\\
                                                        &=\E_{q(\tau)}\left[r(\rvs_t,\rva_t)+f( \cdot )+\gH[\pi^A] + \gH[\pi^O]\right]
\end{align}
where line 2 is substituting \Eqref{eq:maxent_pgm} and
\Eqref{eq:var_hit_mdp} into the ELBO. As a result, the maximum
entropy objective naturally arises in \Eqref{eq:elbo}. Optimizing
the ELBO not only seeks high-reward options but also maximizes
entropy across the space, promoting extensive exploration and
maintaining high diversity.

Given the ELBO, we now define soft value functions of the option
framework following the Bellman Backup Functions along the
trajectory $q(\tau)$ as bellow:
\begin{align}
  \label{eq:elbo_o}
  Q^{soft}_O[\rvs_t,\rvo_t]&= f(\cdot) + \E_{\rva_t \sim \pi^A}\left[Q^{soft}_A[\rvs_t,\rvo_t,\rva_t]\right] + H[\pi^A],\\
  \label{eq:elbo_a}
  Q^{soft}_A[ \rvs_t,\rvo_t,\rva_t]&= r(s,a) + \E_{\rvs_{t+1}\sim P(\rvs_{t+1}|\rvs_t,\rva_t)}\left[\E_{\rvo_{t+1}\sim\pi^O}\left[ Q^{soft}_O[\rvs_{t+1},\rvo_{t+1}]\right] +H[\pi^O]\right]
\end{align}
Assuming policies $\pi^A,\pi^O\in \Pi$ where $\Pi$ is an
arbitrary feasible set, under a tabular setting where the
inference on $\gL$ can be done exactly, we have the following
theorem holds:

\begin{thm}[Soft Option Policy Iteration Theorem]
  \label{thm:soft_mopi}
  Repeated optimizing $\gL$ and $\KL$ defined in
  \Eqref{eq:lme_hitmdp} from any \(\pi^A_0,\pi^O_0 \in \Pi\)
  converges to optimal policies \(\pi^{A*}, \pi^{O*}\) such that
  $Q^{soft*}_O[\rvs_t,\rvo_t] \geq Q^{soft}_O[\rvs_t,\rvo_t]$ and
  $Q^{soft*}_A[ \rvs_t,\rvo_t,\rva_t]\geq Q^{soft}_A[
  \rvs_t,\rvo_t,\rva_t]$, for all \(\pi^A_0,\pi^O_0 \in \Pi\) and
  \((\rvs_t, \rva_t, \rvo_t) \in \mathcal{S} \times \mathcal{A}
  \times \gO \), assuming under tabular settings where
  $|\gS|<\infty, \; |\gO|<\infty, \; |\gA|<\infty$.
  \begin{proof}
    See \appref{app:thm_soft_mopi}.
  \end{proof}
\end{thm}

Theorem~\ref{thm:soft_mopi} guarantees finding the optimal
solution only when the inference can be done exactly under
tabular settings. However, real-world applications often involve
large continuous domains and employ neural networks as function
approximators. In these cases, inference procedures can only be
done approximately. This necessitate a practical approximation
algorithm which we present below.

\subsection{Variational Markovian Option Critic Algorithm}
\label{sec:meth_vmoc}
Formulating complex problems as probabilistic graphical models
(PGMs) allowing us to leverage established methods from PGM
literature to address the associated inference and learning
challenges in real-world applications. To this end, we utilizes
the structured variational inference treatment for optimizing the
log-likelihood of optimal trajectory and prove its convergence
under approximate inference settings. Specifically, using the
variational distribution $q(\tau)$ (\Eqref{eq:var_hit_mdp}) as an
approximator, the ELBO can be derived as (see
\appref{app:deriv_lme_hitmdp}):
\begin{equation}
  \label{eq:lme_hitmdp}
  \gL(q(\tau),P(\tau,\gE^A_{1:T},\gE^O_{1:T}))=- \KL(q(\tau)||P(\tau|\gE^A_{1:T},\gE^O_{1:T}))  +\log P(\gE^A_{1:T},\gE^O_{1:T})
\end{equation}
where $\KL$ is the KL-Divergence between the trajectory following
variational policies $q(\tau)$ and optimal policies
$P(\tau|\gE^A_{1:T},\gE^O_{1:T})$. Under the structural
variational inference~\cite{koller2009probabilistic} perspective,
convergence to the optimal policy can be achieved by optimizing
the ELBO with respect to the the variational policy repeatedly:
\begin{thm}[Convergence Theorem for Variational Markovian Option
  Policy Iteration]
  \label{thm:var_hitmdp}
  Let \(\tau\) be the latent variable and \(\gE^A, \gE^O\) be the
  ground-truth optimality variables. Define the variational
  distribution \(q(\tau)\) and the true log-likelihood of
  optimality \(\log P(\gE^A, \gE^O)\). iterates according to the
  update rule $q^{k+1}=\argmax_q
  \gL(q(\tau),P(\tau,\gE^A_{1:T},\gE^O_{1:T}))$ converges to the
  maximum value bounded by the true log-likelihood of optimality.
  \begin{proof}
    See \appref{app:thm_var_hitmdp}.
  \end{proof}
\end{thm}

We further implements a practical algorithm, the Variational
Markovian Option Critic (VMOC) algorithm, which is suitable for
complex continuous domains. Specifically, we employ parameterized
neural networks as function approximators for both the
Q-functions ($Q_{\psi^A}^{soft}$, $Q_{\psi^O}^{soft}$) and the
policies ($\pi_{\theta^A}$, $\pi_{\theta^O}$). Instead of running
evaluation and improvement to full convergence using
\autoref{thm:soft_mopi}, we can optimize the variational
distribution by taking stochastic gradient descent following
\autoref{thm:var_hitmdp} with respect to the ELBO
(\Eqref{eq:elbo}) directly. Share the same motivation
with~\citename{haarnoja2018soft} of reducing the variance during
the optimization procedure, we derive an option critic framework
by optimizing the maximum entropy objectives between the action
\Eqref{eq:elbo_a} and the option \Eqref{eq:elbo_o} alternatively.
The Bellman residual for the action critic is:
\begin{align*}
  J_{Q^A}(\psi^A_i) = \mathbb{E}&_{(\rvs_t, \rvo_t, \rva_t, \rvs_{t+1}) \sim D} \bigg[ \bigg( \min_{i=1,2}Q_{\psi^A_i}(\rvs_t, \rvo_t, \rva_t) -\\
                                &\big( r(\rvs_t, \rva_t) + \E_{\rvo_{t+1}\sim\pi^O}\left[ Q^{soft}_O[\rvs_{t+1},\rvo_{t+1}]\right] +\alpha^OH[\pi^O]\big)\bigg)^2 \bigg]
\end{align*}
where $\alpha^O$ is the temperature hyper-parameter and the
expectation over option random variable $\E_{\rvo_{t+1}\sim
  \pi^O}$ can be evaluated exactly since $\pi^O$ is a discrete
distribution. The Bellman residual for the option critic is:
\begin{align*}
  J_{Q^O}(\psi^O_i) = \mathbb{E}&_{(\rvs_t, \rvo_t) \sim D} \bigg[ \bigg( \min_{i=1,2}Q^O_{\psi^O_i}(\rvs_t, \rvo_t) -\\
                                &\big(
                                  f(\cdot) + \E_{\rva_t \sim
                                  \pi^A}\left[Q^{soft}_A[\rvs_t,\rvo_t,\rva_t]-\alpha^A\log q(\rva_t|\rvs_t,\rvo_t)\right]
                                  \big) \bigg)^2 \bigg]
\end{align*}
$\alpha^A$ is the temperature hyper-parameter. Unlike
$\E_{\rvo_{t+1}\sim\pi^O}$ can be trivially evaluated, evaluating
$\E_{\rva_t \sim \pi^A}$ is typically intractable. Therefore, in
implementation we use $\rva_t$ sampled from the replay buffer to
estimate the expectation over $\pi^A$.

Following \autoref{thm:var_hitmdp}, the policy gradients can be
derived by directly taking gradient with respect to the ELBOs
defined for the action \Eqref{eq:elbo_a} and the option
\Eqref{eq:elbo_o} policies respectively. The action policy
objective is given by:
\begin{align*}
  J_{\pi^A}(\theta^A)=-\E_{(\rvs_t, \rvo_t) \sim D}\left[\min_{i=1,2}Q_{\psi^A_i}(\rvs_t, \rvo_t, \tilde{\rva}_t) - \alpha^A\log q(\tilde{\rva}_t|\rvs_t,\rvo_t) \right], \; \tilde{\rva}_t\sim q(\cdot|\rvs_t,\rvo_t)
\end{align*}
where in practice the action policy is often sampled by using the
re-parameterization trick introduced in \cite{haarnoja2018soft}.
The option objective is given by:
\begin{align*}
  J_{\pi^O}(\theta^O)=-\E_{(\rvs_t, \rvo_{t-1}) \sim D}\left[\min_{i=1,2}Q_{\psi^O_i}(\rvs_t, \rvo_t) + \alpha^O\gH[\pi^O] \right]
\end{align*}
The variational distribution \( q(\tau) \) defined in
\Eqref{eq:var_hit_mdp} allows us to learn options as
embeddings~\cite{li2020skill,li2022hit} with a learnable
embedding matrix \( \mathbf{W} \in
\mathbb{R}^{\text{num\_options} \times \text{embedding\_dim}} \).
Under this setting, the embedding matrix \( \mathbf{W} \) can be
absorbed into the parameter vector \( \theta^O \). This
integration into VMOC ensures that options are represented as
embeddings without any additional complications, thereby
enhancing the expressiveness and scalability of the model.

The temperature hyper-parameters can also be adjusted by
minimizing the following objective:
\begin{align*}
  J(\alpha^A) = -\E_{\tilde{\rva}_t \sim \pi^A} \left[\alpha^A (\log \pi^A(\tilde{\rva}_t \mid \rvs_t, \rvo_t) + \overline{\gH})\right]
\end{align*}
for the action policy temperature \(\alpha^A\), where
\(\overline{\gH}\) is a target entropy. Similarly, the option
policy temperature \(\alpha^O\) can be adjusted by:
\begin{align*}
  J(\alpha^O) = -\E_{\rvo_t \sim \pi^O} \left[\alpha^O (\log \pi^O(\rvo_t \mid \rvs_t, \rvo_{t-1}) + \overline{\gH})\right]
\end{align*}
where \(\overline{\gH}\) is also a target entropy for the option
policy. In both cases, the temperatures \(\alpha^A\) and
\(\alpha^O\) are updated using gradient descent, ensuring that
the entropy regularization terms dynamically adapt to maintain a
desired level of exploration. This approach aligns with the
methodology proposed in SAC~\cite{haarnoja2018soft}. By adjusting
the temperature parameters, the VMOC algorithm ensures a balanced
trade-off between exploration and exploitation, which is crucial
for achieving optimal performance in complex continuous control
tasks. We summarize the VMOC algorithm in \appref{app:algo}.

\section{Cold-Start Training for Latent Reasoning}
\label{sec:coldstart}

While VMOC provides a principled framework for learning option
embeddings through interaction with the environment, deploying it
for language-based reasoning tasks presents a unique challenge:
how to initialize the latent option space $\mathcal{O}$ to
capture abstract reasoning patterns before environmental
interaction. To address this, we propose a cold-start training
phase using supervised fine-tuning (SFT) data that bridges the
gap between language understanding and latent reasoning.

\subsection{Variational Learning from Demonstrations}

We assume access to a dataset of reasoning demonstrations
$\mathcal{D} = \{(W^{(i)}, Y^{r(i)}, Y^{a(i)})\}_{i=1}^N$, where
$W$ represents the input prompt, $Y^r$ denotes the
chain-of-thought (CoT) reasoning, and $Y^a$ is the final answer.
Our goal is to learn a latent option representation $\rvo \in
\mathcal{O}$ that captures the abstract reasoning process
underlying the transition from prompt to answer.

We formulate this as a variational learning problem where the
option embeddings $\rvo$ serve as latent variables that encode
reasoning strategies. The generative model factorizes as:
\begin{align}
    p(Y^r, Y^a | W) = \int_{\rvo} p(Y^r | \rvo, W) p(Y^a | \rvo, W) p(\rvo | W) d\rvo
\end{align}
where $p(\rvo | W)$ is the prior distribution over option
embeddings given the prompt, $p(Y^r | \rvo, W)$ generates the
reasoning chain conditioned on the latent option, and $p(Y^a |
\rvo, W)$ produces the final answer.

\subsection{Evidence Lower Bound for Language}

Following the variational inference framework established in
Section~\ref{sec:meth}, we introduce a posterior distribution
$q(\rvo | W, Y^r, Y^a)$ and derive the evidence lower bound:
\begin{align}
    \log p(Y^r, Y^a | W) &\geq \mathcal{L}_{\text{SFT}}(W, Y^r, Y^a) \\
    &= \mathbb{E}_{q(\rvo | W, Y^r, Y^a)} \left[ \log p(Y^r | \rvo, W) + \log p(Y^a | \rvo, W) \right] \nonumber \\
    &\quad - \text{KL}(q(\rvo | W, Y^r, Y^a) \| p(\rvo | W))
\end{align}

This objective naturally decomposes into three components:
\begin{itemize}
    \item \textbf{Reasoning Reconstruction}: $\mathbb{E}_q[\log p(Y^r | \rvo, W)]$ ensures the latent option can reconstruct the reasoning process
    \item \textbf{Answer Reconstruction}: $\mathbb{E}_q[\log p(Y^a | \rvo, W)]$ ensures the option produces correct answers
    \item \textbf{KL Regularization}: $\text{KL}(q \| p)$ prevents the posterior from deviating too far from the prior
\end{itemize}

\subsection{Discrete Option Embeddings}

To maintain consistency with VMOC's discrete option space while enabling gradient-based optimization, we parameterize $\rvo$ as a sequence of discrete latent tokens: $\rvo = (\rvo_1, \rvo_2, ..., \rvo_L)$ where each $\rvo_i \in \{1, ..., K\}$ and $K$ is the latent vocabulary size. This allows us to leverage the same option embedding matrix $\mathbf{W} \in \mathbb{R}^{K \times d}$ used in VMOC.

The prior and posterior distributions factorize autoregressively:
\begin{align}
    p(\rvo | W) &= \prod_{i=1}^L p(\rvo_i | \rvo_{<i}, W) \\
    q(\rvo | W, Y^r, Y^a) &= \prod_{i=1}^L q(\rvo_i | \rvo_{<i}, W, Y^r, Y^a)
\end{align}

During training, we use the reparameterization trick with
Gumbel-Softmax to enable backpropagation through the discrete
sampling process.

\subsection{Training Procedure}

The cold-start training uses the SFT dataset to optimize the ELBO
objective in a single forward pass. For each data sample
containing a prompt $W$, a reasoning chain $Y^r$, and an answer
$Y^a$, the model first computes the posterior distribution
$q(\rvo | W, Y^r, Y^a)$. A discrete latent option $\rvo$ is
sampled from this posterior. This sampled option is then used to
compute two reconstruction losses: one for the reasoning chain,
$\log p(Y^r | \rvo, W)$, and one for the final answer, $\log
p(Y^a | \rvo, W)$. Concurrently, the model computes the prior
$p(\rvo|W)$ using only the prompt and calculates the KL
divergence between the posterior and prior. The final loss is a
weighted sum of the reconstruction losses and the KL divergence,
which is optimized to update the model parameters jointly.

\begin{figure}[H]
  \centering
  \includegraphics[width=0.9\linewidth]{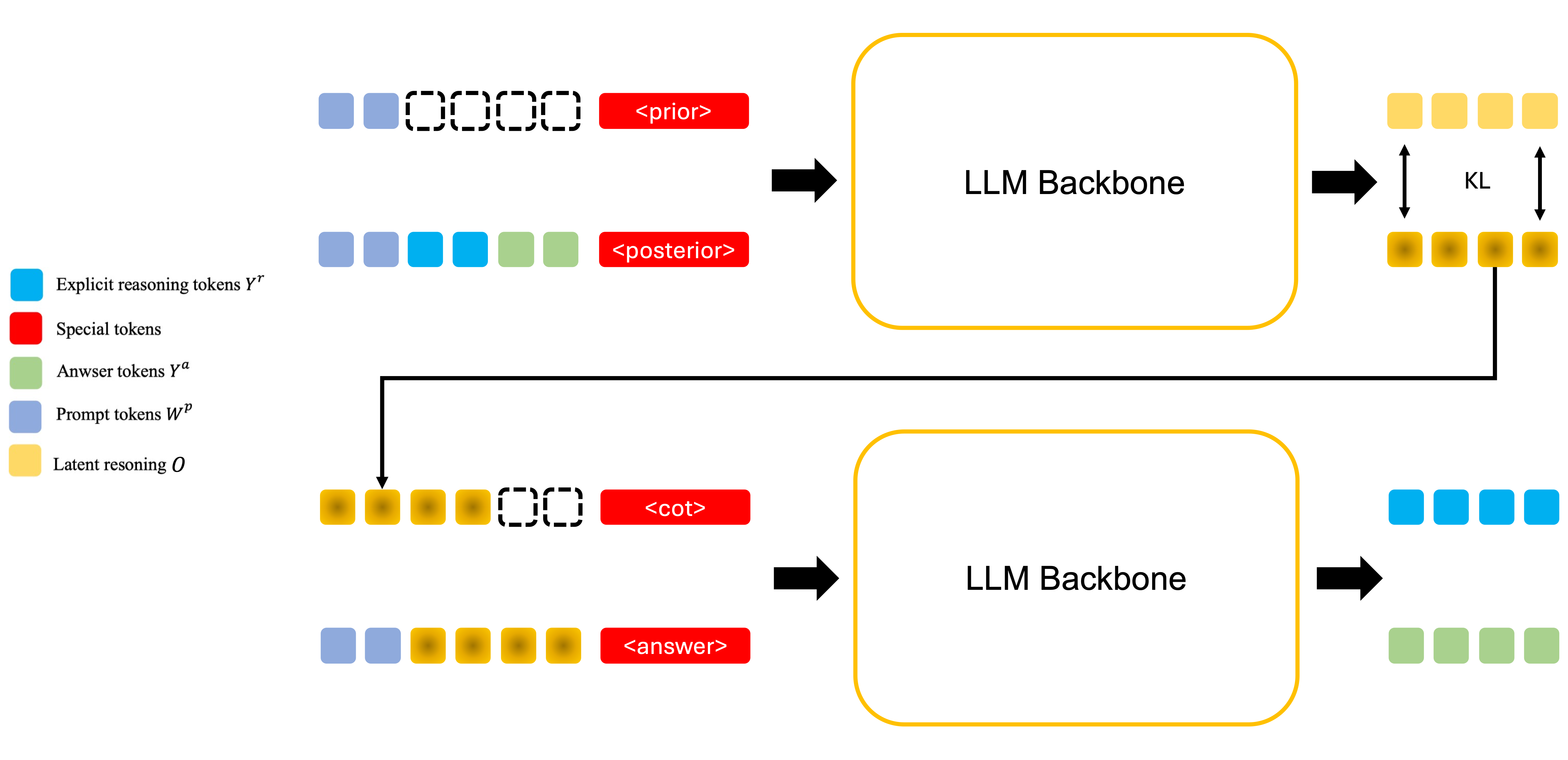}
  \caption{The training process for the cold-start phase. A
    complete data sample (prompt, CoT, answer) is fed into a
    shared encoder to produce a posterior distribution. A latent
    option is sampled and used to decode both the CoT and the
    answer, providing reconstruction signals. The KL divergence
    between the posterior and the prior (generated from the
    prompt only) acts as a regularizer. All components are
    trained jointly with a single ELBO objective.}
\end{figure}

The procedure is outlined in Algorithm 2. During inference, the
model leverages the learned prior to generate reasoning and
answers without access to the ground truth. Given a new prompt
$W$, the model first samples a latent option $\rvo$ from the
prior distribution $p(\rvo|W)$. This latent option, which
encapsulates an abstract reasoning strategy, is then conditioned
upon to first generate the chain-of-thought reasoning $Y^r$ and
subsequently the final answer $Y^a$.

\begin{figure}[H]
  \centering
  \includegraphics[width=0.9\linewidth]{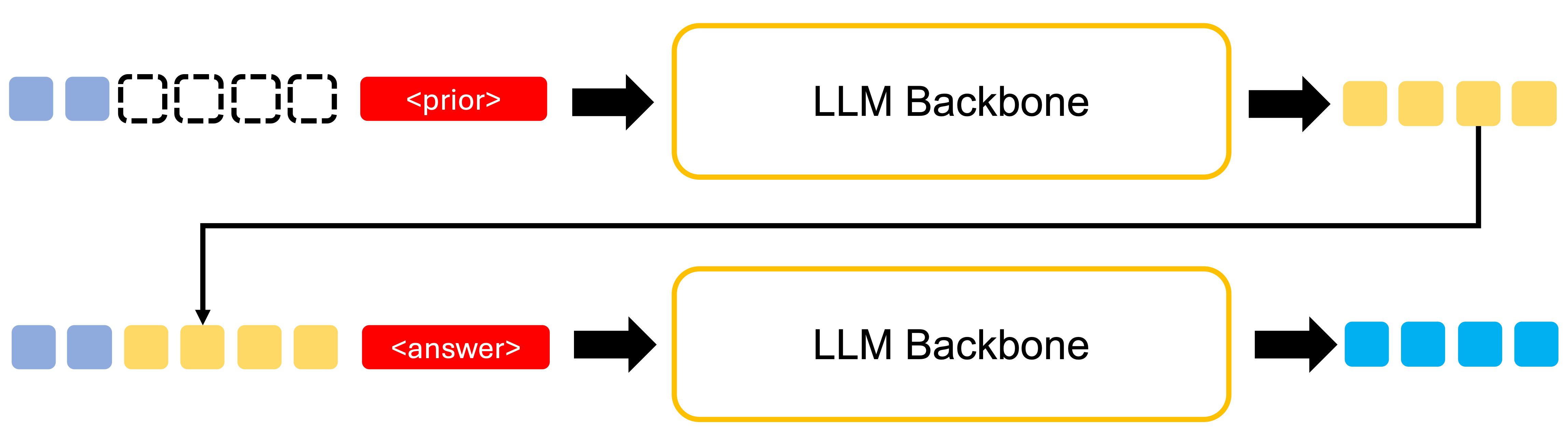}
  \caption{The inference process. With no access to the ground
    truth reasoning or answer, the model first generates a
    distribution over latent options from the prior network given
    only the input prompt. A latent option is sampled and then
    used to autoregressively generate the final answer.}
\end{figure}

\begin{algorithm}[H]
\caption{Cold-Start Training for Latent Reasoning}
\begin{algorithmic}[1]
\STATE \textbf{Input:} SFT dataset $\mathcal{D}$, learning rate $\eta$, KL weight $\beta$
\STATE \textbf{Initialize:} Option embeddings $\mathbf{W}$, encoder/decoder parameters $\theta$
\FOR{each batch $(W, Y^r, Y^a) \sim \mathcal{D}$}
    \STATE Compute posterior $q(\rvo | W, Y^r, Y^a)$ using encoder
    \STATE Sample $\rvo \sim q(\rvo | W, Y^r, Y^a)$ via Gumbel-Softmax
    \STATE Compute prior $p(\rvo | W)$ using encoder with prompt only
    \STATE Compute reconstruction losses:
    \STATE \quad $\mathcal{L}_{\text{CoT}} = -\log p(Y^r | \rvo, W)$
    \STATE \quad $\mathcal{L}_{\text{ans}} = -\log p(Y^a | \rvo, W)$
    \STATE Compute KL divergence: $\mathcal{L}_{\text{KL}} = \text{KL}(q \| p)$
    \STATE Update parameters: $\theta \leftarrow \theta + \eta \nabla_\theta (\mathcal{L}_{\text{CoT}} + \mathcal{L}_{\text{ans}} + \beta \mathcal{L}_{\text{KL}})$
\ENDFOR
\end{algorithmic}
\end{algorithm}

\subsection{Integration with VMOC}

After cold-start training, the learned option embeddings
$\mathbf{W}$ and prior distribution $p(\rvo | W)$ provide a rich
initialization for VMOC. The option policy $\pi^O(\rvo_t |
\rvs_t, \rvo_{t-1})$ in VMOC can be initialized using the learned
prior, where the state $\rvs_t$ includes the current language
context. This enables VMOC to begin with meaningful option
representations that encode diverse reasoning strategies, rather
than starting from random initializations.

The cold-start phase essentially provides VMOC with a "library"
of reasoning patterns encoded in the option space, which can then
be refined through environmental interaction. This two-stage
approach—variational pre-training followed by reinforcement
learning—enables the development of agents that can think
abstractly in latent space while communicating naturally in
language.

\section{Continuous HiT-MDP Homomorphisms}
\label{sec:cont_hitmdp_homomorphism}

In the preceding sections, we introduced VMOC as a practical
algorithm for learning latent options and proposed a cold-start
training phase to initialize these options for implicit reasoning
tasks from demonstrations. This allows an agent to operate in a
simplified, abstract space of latent "thoughts." However, a
critical question remains: what theoretical guarantees do we have
that a policy learned in this abstract option space is sound?
Without a formal grounding, we cannot be sure that solving the
problem in the abstract space is equivalent to solving the
original, more complex problem. To provide this crucial
justification, we now establish the theoretical foundation for
our approach by extending the theory of continuous MDP
homomorphisms~\cite{panangaden2024policy}. Building on recent
work that formalized homomorphisms for standard continuous MDPs,
we develop the first, to our knowledge, continuous homomorphism
for option-based HiT-MDPs. This framework will allow us to prove
that optimal value functions are preserved under our state-option
abstractions, providing a principled guarantee for learning in
the abstract spaces discovered by VMOC.

Based on the fact that HiT-MDP can be re-formulated as and is
equivalent to a continuous MDP~\cite{li2021skill,li2023hit}, we
follow~\cite{panangaden2024policy} and propose the Continuous
HiT-MDP Homomorphism for learning temporal-state-action
abstractions. First, we introduce a construction from
differential topology.

\begin{definition}[Vector bundle]
    Let $M$ be a topological manifold. A \emph{real vector bundle over $M$} (\emph{base space}) is a topological space $E$ (\emph{total space of the bundle}) together with a surjective continuous map $\text{pr}: E \to M$ (bundle projection) satisfying the following conditions:
    \begin{enumerate}
        \item For each $p \in M$, the fiber $E_p = \text{pr}^{-1}(p)$ over $p$ is endowed with the structure of a $k$-dimensional real vector space.($k_p$ may vary as $p$ changes)
        \item For each $p \in M$, there exists a neighborhood $U$ of $p$ in $M$ and a homeomorphism (not homomorphism) $\Phi: \text{pr}^{-1}(U) \to U \times \mathbb{R}^{k_p}$ (called a local trivialization of $E$ over $U$), satisfying the following conditions: 
            \begin{enumerate}
                \item $\text{pr}_U \circ \Phi = \text{pr}$ (where $\text{pr}_U : U \times \mathbb{R}^{k_p} \to U$ is the projection);
                \item for each $q \in U$, the restriction of $\Phi$ to $E_q$ is a vector space isomorphism from $E_q$ to ${q} \times \mathbb{R}^{k_q} \cong \mathbb{R}^{k_q}$. 
            \end{enumerate}
    \end{enumerate}

\end{definition}

Intuitively, a collection of vector spaces, one for each point in $M$; glued together in a way that looks locally like the Cartesian product of $M$ with $\mathbb{R}^n$, but globally can be 'twisted'. 
This local property will be used in the proof of \ref{thm:opt_continuous_HiT-MDP}.

\begin{remark}
    A vector bundle $E$ itself is a topological manifold. Fix a point $e \in E$. There is a chart around $e$ giving the base coordinates and the fiber coordinates of $e$. For details of construction, see (Loring Tu, Differential Geometry). 
\end{remark}

\begin{remark}
    Every fiber $\text{pr}^{-1}({x}$) is a real vector space of finite dimensions and therefore has dimension $k_x$ and local trivializations show that the function $x \mapsto k_x$ is locally constant and therefore is constant on each component of $X$.
\end{remark}

Now, we are ready to define continuous HiT-MDPs and state our underlying assumptions.

\begin{definition}[Continuous HiT-MDP] A \emph{continuous Hidden Temporal Markov decision process (Hit-MDP)} is a $7$-tuple:
  \[\mathcal{M} = (\mathcal{S}, E_{\mathcal{O}},\mathcal{B},\mathcal{A} \x \mathcal{O},\forall (a, o)\in \mathcal{A} \x \mathcal{O}\;\; \tau_{(a, o)}:\mathcal{S}\x\mathcal{B}\to[0,1],R:\mathcal{S}\x \mathcal{A} \to \R, \gamma),\] where $\mathcal{S}$ is our state space, assumed to be a topological manifold, $E_{\mathcal{O}}$ is a vector bundle over $\mathcal{S}$, in which $\mathcal{O}$, the option space, is a family of vector spaces (of different dimensions), as fibers. $\mathcal{S}$ is assumed to be a topological manifold, $\mathcal{B}$ is the Borel $\sigma$-algebra on $E_{\mathcal{O}}$, $\mathcal{A} \x \mathcal{O}$ is the augmented space of \emph{actions} and space of options, which is a
  locally compact Hausdorff space, $\tau_{(a, o)}$ is
  the transition probability kernel for each possible action-policy pair $(a, o)$, for each fixed $e \in E_\mathcal{O}$,
  $\tau_{(a, o)}(\cdot|e)$ is a probability distribution on $\mathcal{B}$ while $R$ is the reward
  function, and $\gamma$ is the discount factor. 
\end{definition}

\begin{remark}
    It is natural to assume that $\mathcal{O}$ (option space) and $\mathcal{S}$ (state space) have an intrinsic vector bundle structure $E_{\mathcal{O}}$ with $\mathcal{S}$ as base space and $\mathcal{O}$ as fiber. Each learning scenario naturally corresponds to a connected component of the state manifold $\mathcal{S}$. For each component $k$, we give a different  
option vector space $\mathcal{O}_k$ as its fiber space. Assembling the components together, we obtain a vector bundle $E_\mathcal{O}$ such that for each $s \in \mathcal{S}$ in a component $k$, the fiber at $s$ is $E_s = \text{pr}^{-1}(s) = (s, \mathcal{O}_k)$. 
\end{remark}

\begin{remark}
    $E_\mathcal{O}$ is a topological manifold, which is of course a topological space, just as $\mathcal{S}$ in the corresponding definition in ~\cite{panangaden2024policy}. $\mathcal{A} \x \mathcal{O}$, the augmented space, is endowed with the product topology; Since both $\mathcal{A}$ and $\mathcal{O}$ are locally compact, the augmented space is locally compact; $\mathcal{A} \x \mathcal{O}$ is Hausdorff since both $\mathcal{A}$ and $\mathcal{O}$ are. Follow ~\cite{panangaden2024policy}, we have the corresponding existence of policy lifting (\ref{prop:Hit_MDP_lift_existence}) easily.
\end{remark}

Before we introduce our HiT-MDP homomorphism, we would like to introduce the concept of bundle maps to represent the relation between the state-option spaces. 

\begin{definition}[Bundle maps]
    Let $\text{pr}_E: E \to M$ and $\text{pr}_F: F \to N$ be vector bundles. In this paper, a bundle map from $E$ to $F$ is a pair of continuous maps ($\varphi_\text{bundle} : E \to F$, $\varphi_\text{base}: M \to N$) such that the following diagram commutes:
    \[
        \begin{tikzcd}
            E \arrow[r, "\varphi_\text{bundle}"] \arrow[d, "\text{pr}_E"'] & F \arrow[d, "\text{pr}_F"] \\
            M \arrow[r, "\varphi_\text{base}"'] & N
        \end{tikzcd}
    \]
\end{definition}

Abusing the language, we often call the map $\varphi =
(\varphi_\text{bundle}, \varphi_\text{base}): E \to F$ alone the
bundle map. Now we are ready to define our continuous HiT-MDP
homomorphism here.

\begin{definition}[Continuous HiT-MDP Homomorphism]
\label{def: Continuous HiT-MDP Homomorphism}
    A \emph{continuous HiT-MDP homomorphism} is a map $h = ( f, g_e ): \mathcal{M} \to \overline{\mathcal{M}}$ where $f: E_{\mathcal{O}} \to \overline{E}_{\overline{{\mathcal{O}}}}$ is a bundle map and for every $e \in  E_{\mathcal{O}}$, $g_e: \mathcal{A} \x \mathcal{O} \to \overline{\mathcal{A}} \x \overline{\mathcal{O}}$ are Borel-measurable, surjective maps such that the following hold:\\
    \begin{enumerate}
        \item Invariance of reward: $ \overline{R}(f(e), g_e(a, o)) = R(e,(a, o)) \qquad \forall e \in E_{\mathcal{O}}, (a, o) \in \mathcal{A} \x \mathcal{O} $ 
        \item Equivariance of transitions: fix any $e \in E_\mathcal{O}, (a, o) \in \mathcal{A} \x \mathcal{O}$, \\
        $\overline{\tau}_{g_e(a, o)}(\overline{B}| f(e)) = \tau_{(a, o)}(f^{-1}(\overline{B})| s)  \qquad \forall \; \overline{B} \in \overline{\mathcal{B}}$
    \end{enumerate}
\end{definition}

\begin{remark}
    In this definition, our definition of a bundle map $f$ can be illustrated in the following diagram:
    \[
        \begin{tikzcd}
            E_{\mathcal{O}} \arrow[r, "f_\text{bundle}"] \arrow[d, "\text{pr}_S"'] & \overline{E}_{\overline{\mathcal{O}}} \arrow[d, "\text{pr}_{\overline{S}}"] \\
            S \arrow[r, "f_\text{base}"'] & \overline{S}
        \end{tikzcd}
    \]
    Our definition implies that our map $f$ should not only satisfy Definition~\ref{def: Continuous HiT-MDP Homomorphism} but also satisfy this diagram. The philosophy behind this is that given an element $e \in E_\mathcal{O}$, which locally could be interpreted as a state-option tuple $(s, o) \in \mathcal{S} \times \mathcal{O}$, there are two ``paths'' to the image HiT-MDP. Path $\text{pr}_{\overline{S}} \circ f_\text{bundle}$: We can encode this tuple through $f$ and it becomes $(\overline{s}, \overline{o})$; this abstraction then contains the abstraction of $s$ and can be projected onto $\overline{S}$ by $\text{pr}_{\overline{S}}$ to get $\overline{s}$, path $f_\text{base} \circ \text{pr}_S$: we encode $s$ directly through $\underline{f}$. These two encoding paths should end at the same destination. This will also be used in the proof of Theorem~\ref{thm:opt_continuous_HiT-MDP}.
\end{remark}

\begin{remark}
    The equivariance of transitions is defined using the Borel $\sigma$-algebra defined on the image MDP; it states that, for any fixed $e \in E_\mathcal{O}, (a, o) \in \mathcal{A} \x \mathcal{O}$, the measure $\overline{\tau}_{g_e(a,o)}(\cdot| f(e))$ is the pushforward measure of $\tau_{(a, o)}(\cdot | e)$ on $E_{\mathcal{O}}$ under the bundle map $f$. 
\end{remark}

\subsection{Optimal Value Equivalence}

\begin{theorem}[Optimal Value Equivalence in HiT-MDP]
\label{thm:opt_continuous_HiT-MDP}
Let $\overline{\mathcal{M}} = (\overline{\mathcal{S}}, \overline{E}_{\overline{\mathcal{O}}}, \overline{\mathcal{B}}, \overline{\mathcal{A}} \x \overline{\mathcal{O}}, \overline{\tau}_{\overline{(a, o)}}, \overline{R}, \overline{\gamma})$ be the image of a continuous MDP homomorphism $h = (f, g_e)$ from $\mathcal{M} = (\mathcal{S}, E_\mathcal{O}, \mathcal{B}, \mathcal{A} \x \mathcal{O}, \tau_{(a, o)}, R, \gamma)$.  Then for any $(e, a, o) \in E_\mathcal{O} \times \mathcal{A} \times \mathcal{O}$ we have:
$$Q^*(e,a,o) = \overline{Q}^*(f(e), g_e(a,o)),$$
where $Q^*,  \overline{Q}^*$ are the optimal action-value functions for $\mathcal{M}$ and $\overline{\mathcal{M}}$, respectively.
\end{theorem}

\begin{proof}
We prove the equivalence of the value functions under a continuous HiT-MDP homomorphism. The homomorphism is defined by bundle map $f: E_{\mathcal{O}} \to \overline{E}_{\overline{{\mathcal{O}}}}$ and, for each $e \in E_{\mathcal{O}}$, a surjective map $g_e: \mathcal{A} \times \mathcal{O} \to \overline{\mathcal{A}} \times \overline{\mathcal{O}}$. 

Since $E_\mathcal{O}$ is a vector bundle, locally we can write our full system state as $e_t = (s_t, o_{t-1})$ For simplicity and clarity, we use $e_{t}$ and $(s_t, o_{t-1})$ interchangeably in derivations below. We denote an action as the pair $\alpha_t = (a_t, o_t)$.  Due to conditional independencies introduced in~\cite{li2023hit}, the policy over primitive actions typically depends on the current state and current option, $\pi(a_t|s_t, o_t)$, and the option policy depends on the current state and previous option, $P(o_t|s_t, o_{t-1})$. Therefore, we have

$$Q_n (e_t, a_t, o_t) = Q_m(s_t, o_{t-1}, a_t, o_t) = Q_m(s_t, a_t, o_t).$$

We will prove the following claim by induction.

\begin{claim}
Define the sequence of value functions $Q_m$: 
$$Q_m(s_t, o_{t-1}, a_t, o_t) = R(s_t, a_t) + \gamma \int_{s_{t+1}} P(ds_{t+1}|s_t, a_t) V_{m-1}(s_{t+1}, o_t)$$
where the state-value function $V^\pi_{m-1}$ is the expected action-value at the next state $(s_{t+1}, o_t)$:
$$V_{m-1}(s_{t+1}, o_t) = \mathbb{E}_{\pi(\cdot|s_{t+1}, o_t)}[Q_{m-1}(s_{t+1}, o_t, \cdot)]$$
This expectation expands to:
$$V_{m-1}(s_{t+1}, o_t) = \int_{o_{t+1}} P(do_{t+1}|s_{t+1}, o_t) \int_{a_{t+1}} \pi(da_{t+1}|s_{t+1}, o_{t+1}) Q_{m-1}(s_{t+1}, o_t, a_{t+1}, o_{t+1})$$
Let $Q_0(s_t, o_{t-1}, a_t, o_t) = R(s_t, a_t)$. Let $\overline{Q}_m$ be the analogously defined sequence for the abstract MDP $\overline{\mathcal{M}}$ and policy $\overline{\pi}$. Then for any state-action tuple, we claim:
$$Q_m(s_t, o_{t-1}, a_t, o_t) = \overline{Q}_m(f(s_t, o_{t-1}), g_{e_t}(a_t, o_t)).$$
\end{claim}

We prove this claim by induction on the iteration step $m$.

\paragraph{Base Case ($m=0$):}
The base case follows directly from the reward invariance property of the homomorphism:
$$Q_0(s_t, o_{t-1}, a_t, o_t) = R(s_t, a_t) = \overline{R}(f(s_t, o_{t-1}), g_{(s_t, o_{t-1})}(a_t, o_t)) = \overline{Q}_0(f(s_t, o_{t-1}), g_{(s_t, o_{t-1})}(a_t, o_t)).$$

\paragraph{Inductive Step:}
Assume the claim holds for iteration $m-1$. We show it holds for iteration $m$.

\begin{align}
    Q_m&(s_t, o_{t-1}, a_t, o_t) \\
    = & R(s_t, a_t) + \gamma \int_{e_{t+1}\in E_{\cO}} P(de_{t+1}|s_t, a_t) V_{m-1}^*(s_{t+1}, o_t) \\
    = & R(s_t, a_t) + \gamma \int_{e_{t+1}\in E_{\cO}} P(de_{t+1}|s_t, a_t)  \sup_{(a_{t+1}, o_{t+1})\in \cA \x \cO } Q_{m-1}(s_{t+1}, o_{t+1}, a_{t+1}) \\
    = & \overline{R}(f(s_t, o_{t-1}), g_{(s_t, o_{t-1})}(a_t, o_t)) \nonumber\\
  & + \gamma \int_{e_{t+1}\in E_{\cO}} P(de_{t+1}|s_t, a_t) \sup_{(a_{t+1}, o_{t+1})\in \cA \x \cO } \overline{Q}_{m-1}(f(e_{t+1}), g_{e_t}(a_{t+1}, o_{t+1})) \label{def_Q} \\
       = & \overline{R}(f(s_t, o_{t-1}), g_{(s_t, o_{t-1})}(a_t, o_t))  \nonumber\\
  &+ \gamma \int_{e_{t+1}\in E_{\cO}} P(de_{t+1}|s_t, a_t) \sup_{(\bar{a}_{t+1}, f_{\text{bundle}}(o_{t+1}))\in \bar{\cA} \x \bar{\cO} } \overline{Q}_{m-1}(f(e_{t+1}), \bar{a}_{t+1}, f_{\text{bundle}}(o_{t+1})) \label{surjective_part} \\
       = & \overline{R}(f(s_t, o_{t-1}), g_{(s_t, o_{t-1})}(a_t, o_t))  \nonumber\\
  & + \gamma \int_{e_{t+1}\in \overline{E_\cO}}{ P(de_{t+1}|e_t, \bar{a}_t) \sup_{(\bar{a}_{t+1}, \bar{o}_{t+1})\in \bar{\cA} \x \bar{\cO} } \overline{Q}_{m-1}(e_{t+1}, \bar{a}_{t+1}, \bar{o}_{t+1})} \label{apply_CVF_Q} \\
    = & \overline{Q}_m(f(e_{t+1}), g_{e_t}(a_t, o_t))
\end{align}

\begin{enumerate}
    \item (\ref{def_Q}) follows from applying the inductive hypothesis to $Q^\pi_{m-1}$ inside the expectation.
    \item (\ref{surjective_part}) is the crucial step. It follows from the fact that $g_{e_t}$ is surjective. 
    
    \item (\ref{apply_CVF_Q}) follows from the change of variables formula (transition equivariance in the definition \ref{def: Continuous HiT-MDP Homomorphism} and the fact that $f$ is a bundle map. Here, the pushforward probability distribution of $P(de_{t+1}|s_t, a_t)$ is equal to $P(de_{t+1}|e_t, \bar{a}_t)$ and the integrand is a function from $E_\cO \to \mathbb{R}$ in the form $e_{t+1} \to 
    \sup_{(\bar{a}_{t+1}, f_{\text{bundle}}(o_{t+1}))\in \bar{\cA} \x \bar{\cO} } \overline{Q}_{m-1}(f(e_{t+1}), \bar{a}_{t+1}, f_{\text{bundle}}(o_{t+1}))$. 
    
    \item The final line is the definition of $\overline{Q}^{\overline{\pi}}_m$.
\end{enumerate}

This concludes the induction. \\

Since iterative policy evaluation converges to the true value function of the policy, $Q^\pi_m \to Q^\pi$ as $m \to \infty$. It follows that for any $(s_t, o_{t-1}, a_t, o_t)$,
$$Q^\pi(s_t, o_{t-1}, a_t, o_t) = \overline{Q}^{\overline{\pi}}(f(s_t, o_{t-1}), g_{(s_t, o_{t-1})}(a_t, o_t)).$$
By following the conditional independence mentioned above: 
$$Q^*(e_t, a_t, o_t) = \overline{Q}^*(f(e_t), g_{e_t}(a_t, o_t)).$$
\end{proof}

Just as in \cite{panangaden2024policy}, we can derive the existence of the lifted policies via an HiT-MDP homomorphism. We will formalize our lifted policy for continuous MDP homomorphism as follows: 

\begin{definition}[HiT-MDP Policy Lifting]
\label{def:HiT-MDP-liftedpolicy}
Let $\overline{\mathcal{M}} = (\overline{\mathcal{S}}, \overline{E}_{\overline{\mathcal{O}}}, \overline{\mathcal{B}}, \overline{\mathcal{A}} \x \overline{\mathcal{O}}, \overline{\tau}_{\overline{(a, o)}}, \overline{R}, \overline{\gamma})$ be the image of a continuous MDP homomorphism $h = (f, g_e)$ from $\mathcal{M} = (\mathcal{S}, E_\mathcal{O}, \mathcal{B}, \mathcal{A} \x \mathcal{O}, \tau_{(a, o)}, R, \gamma)$. Then for any policy $\overline{\pi} : \overline{\mathcal{S}} \to \Dist(\overline{\mathcal{A}} \x \overline{\mathcal{O}})$ defined on $\overline{\mathcal{M}}$, a policy $\pi^\uparrow : E_\mathcal{O} \to \Dist(\mathcal{A} \x \mathcal{O})$ on $\mathcal{M}$ is a \emph{lifted} policy of $\overline{\pi}$ if:
\begin{equation}
    \pi^\uparrow(g_e^{-1}(\beta) | e) = \overline{\pi}(\beta | f(e))
    \label{eq:policy_lifting}
\end{equation}
for every Borel set $\beta \in \overline{\mathcal{B}}$ and $e \in E_\mathcal{O}$. In other words, fix any $e \in E_\mathcal{O}$, $\overline{\pi}(f(e), \cdot)$ is the pushforward measure of  $\pi^\uparrow(e, \cdot)$ on the space $\overline{\mathcal{A}} \x \overline{\mathcal{O}}$,  for all  with respect to $g_e$.
\end{definition}

\begin{proposition}
\label{prop:Hit_MDP_lift_existence}
   Let $\overline{\mathcal{M}} = (\overline{\mathcal{S}}, \overline{E}_{\overline{\mathcal{O}}}, \overline{\mathcal{B}}, \overline{\mathcal{A}} \x \overline{\mathcal{O}}, \overline{\tau}_{\overline{(a, o)}}, \overline{R}, \overline{\gamma})$ be the image of a continuous MDP homomorphism $h = (f, g_e)$ from $\mathcal{M} = (\mathcal{S}, E_\mathcal{O}, \mathcal{B}, \mathcal{A} \x \mathcal{O}, \tau_{(a, o)}, R, \gamma)$, where $\mathcal{A} \x \mathcal{O}$ and $\overline{\mathcal{A}} \x \overline{\mathcal{O}}$ are locally compact Hausdorff spaces. Then for any policy $\overline{\pi} : \overline{E_\mathcal{O}} \to \Dist(\overline{\mathcal{A}} \times \overline{\mathcal{O}})$ defined on $\overline{\mathcal{M}}$, there exists a lifted policy $\pi^\uparrow: E_\mathcal{O} \to \Dist(\mathcal{A} \x \mathcal{O})$ in the sense of Definition~\ref{def:HiT-MDP-liftedpolicy}.
\end{proposition}
\begin{proof}
The difference between our proposition and the one in \cite{panangaden2024policy} is here we only require the spaces $\mathcal{A}$, $\mathcal{O}$ etc., to be Hausdorff instead of a metric space, but these changes do not affect the application of Hahn-Banach theorem and Riesz Representation theorem.

As follows: similarly we could define a functional $L: C_0(\cA \x \cO) \to \mathbb{R}$ as: 
$$L(\psi) = \max_{a \in \cA \x \cO} \psi(a)$$. 

We can prove $L$ is a semi-norm. Now $g_e$ is surjective, we define $S := \{\eta \circ g_e : \eta \in C_0(\overline{\cA} \x \overline{\cO})\} \subset C_0(\cA \x \cO)$. Let $T$ be the linear functional on $S$ defined as:

$$T(\eta \circ g_e)  = \int_{b \in \overline{\cA} \x \overline{\cO}} \eta(b) \overline{\pi}(db | f(s)).$$

We can show $T$ is bounded by $L$ on $S \subset C_0(\cA \x \cO)$. By the Hahn-Banach theorem, we can extend $T$ to a linear functional $\overline{T}$ and $\overline{T}$ is still bounded by $L$ on $C_0(\cA \x \cO)$. Thus, $\overline{T}$ can be shown to be a positive linear functional. With our assumption that $\cA \x \cO$ is a locally compact Hausdorff space, we could employ Riesz Representation theorem and obtain a unique regular Borel measure on the space $\mu$ on $\mathcal{A} \x \mathcal{O}$.

The pushforward measure of $\mu$ with respect to $g_e$ will be $\overline{\pi}(\cdot | f(e))$. Setting $\pi^\uparrow(\cdot | e) = \mu$ gives the result. 
\end{proof}

Now that we have proven a lifted policy exists for continuous setting, we proceed to prove a value equivalence result for continuous HiT-MDP homomorphisms. 

\begin{theorem}[HiT-MDP Value Equivalence]
\label{thm:value_equivalence}
 Let $\overline{\mathcal{M}} = (\overline{\mathcal{S}}, \overline{E}_{\overline{\mathcal{O}}}, \overline{\mathcal{B}}, \overline{\mathcal{A}} \x \overline{\mathcal{O}}, \overline{\tau}_{(\overline{a}, \overline{o})}, \overline{R}, \overline{\gamma})$ be the image of a continuous MDP homomorphism $h = (f, g_s)$ from $\mathcal{M} = (\mathcal{S}, E_\mathcal{O}, \mathcal{B}, \mathcal{A} \x \mathcal{O}, \tau_{(a, o)}, R, \gamma)$, and let $\pi^\uparrow$ be a lifted policy corresponding to $\overline{\pi}$. Then for any $(e, a, o) \in E_\mathcal{O} \times \mathcal{A} \times \mathcal{O}$ we have:
$$
    Q^{\pi^\uparrow}(e, a, o) = \overline{Q}^{\overline{\pi}}(f(e), g_e(a, o)),
$$
where $Q^{\pi^\uparrow}(e, a, o)$ and $\overline{Q}^{\overline{\pi}}(f(e), g_e(a, o))$ are the action-value functions for policies $\pi^\uparrow$ and $\overline{\pi}$ respectively.
\end{theorem}
\begin{proof}
    Similarly as in Theorem~\ref{thm:opt_continuous_HiT-MDP}, we define the sequence $Q^{\pi^\uparrow}_m: \mathcal{S} \times \mathcal{A} \to \R$ as:
    
\begin{align*}
    Q^{\pi^\uparrow}_m(s_t, a_t, o_t) = &R(s_t, a_t) + \\
    &\gamma \int_{s_{t+1}} P(ds_{t+1}|s_t, a_t)\int_{a_{t+1},o_{t+1}} P(da_{t+1},do_{t+1}|s_{t+1}, o_t) Q^{\pi^\uparrow}_{m-1}(s_{t+1}, a_{t+1}, o_{t+1})
\end{align*}

    for $m \geq 1$ and $Q^{\pi^\uparrow}_0(s,a,o) = 0$. Analogously define $\overline{Q}_{m-1}^{\overline{\pi}}: \overline{\mathcal{S}} \times \overline{\mathcal{A}} \x \overline{\mathcal{O}}  \to \R$. For the inductive case, we can perform change of variables twice to change the domain of integration from $S$ to $\overline{\mathcal{S}}$ and $\mathcal{A} \x \cO $ to $\overline{\mathcal{A}\x \cO}$ respectively:
    \begin{align*}
        Q^{\pi^\uparrow}_m(s,a,o) &= R(s_t, a_t) + \\
    &\gamma \int_{s_{t+1}} P(ds_{t+1}|s_t, a_t)\int_{a_{t+1},o_{t+1}} P(da_{t+1},do_{t+1}|s_{t+1}, o_t) Q^{\pi^\uparrow}_{m-1}(s_{t+1}, a_{t+1}, o_{t+1}) \\
    =&R(s_t, a_t) + \\
    &\gamma \int_{s_{t+1}} P(ds_{t+1}|s_t, a_t)\int_{a_{t+1},o_{t+1}} P(da_{t+1},do_{t+1}|s_{t+1}, o_t) Q^{\bar{\pi}}_{m-1}(f(s_{t+1},o_t), g(a_{t+1}, o_{t+1})) \\
    =&\bar{R}(\bar{s}_t, \bar{a}_t) + \gamma \int_{s_{t+1}} P(ds_{t+1}|s_t, a_t)\int_{\bar{a}_{t+1},\bar{o}_{t+1}} \bar{\pi} Q^{\bar{\pi}}_{m-1}(f(s_{t+1},o_t), g(a_{t+1}, o_{t+1})) \\
    = & \overline{Q}_{m-1}^{\overline{\pi}}(f(e), g_e(a,o)).
    \end{align*}
    In a similar manner to Theorem~\ref{thm:opt_continuous_HiT-MDP}, we conclude that  $Q^{\pi^\uparrow}(e, a, o) = \overline{Q}^{\overline{\pi}}(f(e), g_e(a, o))$.
\end{proof}

Thus, we have recovered all desirable properties for continuous HiT-MDP homomorphisms corresponding to continuous MDP homomorphisms in \cite{panangaden2024policy}.

\section{Variational Inference on Abstract HiT-MDPs}
\label{sec:var_abstract_hitmdp}

In Section~\ref{sec:meth}, we formulated the learning problem in the original HiT-MDP $\mathcal{M}$ as a variational inference problem, yielding an evidence lower bound (ELBO) whose optimization leads to the VMOC algorithm. In Section~\ref{sec:cont_hitmdp_homomorphism}, we established that a continuous HiT-MDP homomorphism $h=(f, g_e)$ preserves the optimal value function between the original MDP $\mathcal{M}$ and an abstract MDP $\overline{\mathcal{M}}$. In this section, we bridge these two concepts by deriving the ELBO for the abstract HiT-MDP, $\overline{\gL}$, and proving that maximizing $\overline{\gL}$ with respect to the abstract policy $\overline{\pi}$ is equivalent to maximizing the original ELBO $\gL$ with respect to the corresponding lifted policy $\pi^\uparrow$. This provides a principled justification for using VMOC to optimize policies in the abstract space, as it directly optimizes the variational objective in the original, more complex space.

\subsection{The Evidence Lower Bound for Abstract HiT-MDPs}

We begin by defining the control-as-inference problem for the abstract MDP $\overline{\mathcal{M}} = (\overline{\mathcal{S}}, \overline{E}_{\overline{{\mathcal{O}}}}, \overline{\mathcal{B}}, \overline{\mathcal{A}} \x \overline{\mathcal{O}}, \overline{\tau}, \overline{R}, \gamma)$. Let $\overline{e}_t = f(e_t)$ be the abstract state and $\overline{\alpha}_t = g_{e_t}(\alpha_t)$ be the abstract action, where $e_t = (s_t, o_{t-1})$ and $\alpha_t = (a_t, o_t)$. The abstract trajectory is denoted by $\overline{\tau} = (\overline{e}_0, \overline{\alpha}_0, \overline{e}_1, \overline{\alpha}_1, \ldots)$.

We introduce abstract optimality variables $\overline{\gE}^A$ and $\overline{\gE}^O$. The joint probability over an abstract trajectory $\overline{\tau}$ and the optimality variables is given by:
\begin{equation}
    P(\overline{\tau}, \overline{\gE}^A_{1:T}, \overline{\gE}^O_{1:T}) \propto P(\overline{e}_0) \prod_{t=1}^T P(\overline{e}_{t+1} | \overline{e}_t, \overline{\alpha}_t) P(\overline{\gE}^A_t | \overline{e}_t, \overline{\alpha}_t) P(\overline{\gE}^O_t | \overline{e}_t, \overline{\alpha}_t)
\end{equation}
where the likelihood of optimality is defined by the abstract reward function $\overline{R}$ and an abstract regularizer $\overline{f}$, consistent with the original MDP:
\begin{align}
    P(\overline{\gE}^A_t | \overline{e}_t, \overline{\alpha}_t) &= \exp(\overline{R}(\overline{e}_t, \overline{\alpha}_t)) \\
    P(\overline{\gE}^O_t | \overline{e}_t, \overline{\alpha}_t) &= \exp(\overline{f}(\overline{e}_t, \overline{\alpha}_t))
\end{align}
The variational distribution to approximate the true posterior $P(\overline{\tau}|\overline{\gE}^A_{1:T}, \overline{\gE}^O_{1:T})$ is defined by an abstract policy $\overline{q}(\overline{\alpha}_t | \overline{e}_t)$:
\begin{equation}
    \overline{q}(\overline{\tau}) = P(\overline{e}_0) \prod_{t=1}^T P(\overline{e}_{t+1} | \overline{e}_t, \overline{\alpha}_t) \overline{q}(\overline{\alpha}_t | \overline{e}_t)
\end{equation}
The ELBO for the abstract MDP, denoted $\overline{\gL}$, is derived by maximizing the log-likelihood $\log P(\overline{\gE}^A_{1:T}, \overline{\gE}^O_{1:T})$:
\begin{align}
    \overline{\gL}(\overline{q}(\overline{\tau})) &= \E_{\overline{q}(\overline{\tau})} \left[ \log P(\overline{\tau}, \overline{\gE}^A_{1:T}, \overline{\gE}^O_{1:T}) - \log \overline{q}(\overline{\tau}) \right] \\
    &= \E_{\overline{q}(\overline{\tau})} \left[ \sum_{t=1}^T \left( \overline{R}(\overline{e}_t, \overline{\alpha}_t) + \overline{f}(\overline{e}_t, \overline{\alpha}_t) - \log \overline{q}(\overline{\alpha}_t | \overline{e}_t) \right) \right] \\
    &= \E_{\overline{q}(\overline{\tau})} \left[ \sum_{t=1}^T \left( \overline{R}(\overline{e}_t, \overline{\alpha}_t) + \overline{f}(\overline{e}_t, \overline{\alpha}_t) + \gH(\overline{q}(\cdot | \overline{e}_t)) \right) \right]
\end{align}
This abstract ELBO, $\overline{\gL}$, provides the objective function for the VMOC algorithm when applied to the abstract MDP $\overline{\mathcal{M}}$.

\subsection{Optimizing Variational Objectives}

We now show that optimizing the abstract ELBO $\overline{\gL}$ with an abstract policy $\overline{q}$ is equivalent to optimizing the original ELBO $\gL$ with the corresponding lifted policy $q^\uparrow$.

\begin{theorem}[ELBO Equivalence under Homomorphism]
Let $h=(f, g_e): \mathcal{M} \to \overline{\mathcal{M}}$ be a
continuous HiT-MDP homomorphism. Let $\overline{q}$ be a policy
on $\overline{\mathcal{M}}$ and $q^\uparrow$ be its lifted policy
on $\mathcal{M}$ as per
Definition~\ref{def:HiT-MDP-liftedpolicy}. Then the ELBO of
the abstract policy is a lower bound of the lifted policy:
\begin{equation}
    \gL(q^\uparrow) \ge \overline{\gL}(\overline{q})
\end{equation}
\end{theorem}

\begin{proof}
The ELBO for the original MDP $\mathcal{M}$ under the lifted policy $q^\uparrow$ is:
\begin{equation}
    \gL(q^\uparrow) = \E_{q^\uparrow(\tau)} \left[ \sum_{t=1}^T \left( R(e_t, \alpha_t) + f(e_t, \alpha_t) + \gH(q^\uparrow(\cdot | e_t)) \right) \right]
\end{equation}
We analyze each term in the expectation.

\textbf{Reward Term:} By the reward invariance property of the homomorphism, $R(e_t, \alpha_t) = \overline{R}(f(e_t), g_{e_t}(\alpha_t)) = \overline{R}(\overline{e}_t, \overline{\alpha}_t)$. The same holds for the regularizer $f$. Using the change of variables formula and the definition of the lifted policy, the expectation over the reward transforms as follows:
\begin{align}
    \E_{q^\uparrow(\tau)}[R(e_t, \alpha_t)] &= \int_{e_t, \alpha_t} q^\uparrow(\tau) R(e_t, \alpha_t) \,d\tau \\
    &= \int_{\overline{e}_t, \overline{\alpha}_t} \overline{q}(\overline{\tau}) \overline{R}(\overline{e}_t, \overline{\alpha}_t) \,d\overline{\tau} = \E_{\overline{q}(\overline{\tau})}[\overline{R}(\overline{e}_t, \overline{\alpha}_t)]
\end{align}

\textbf{Entropy Term:} The entropy of the lifted policy
$\gH(q^\uparrow(\cdot|e_t))$ can be decomposed using the chain
rule for entropy. Let $Y=g_{e_t}(\alpha_t)$ be a random variable
representing the abstract action. The entropy of $\alpha_t$ can
be written as:
\begin{equation}
    \gH(q^\uparrow(\alpha_t | e_t)) = \gH(g_{e_t}(\alpha_t) | f(e_t)) + \E_{q^\uparrow(\alpha_t|e_t)}[\gH(q^\uparrow(\alpha_t | e_t, g_{e_t}(\alpha_t)))]
\end{equation}
The first term is the entropy of the abstract policy, since $\overline{q}(\cdot|f(e_t))$ is the pushforward measure of $q^\uparrow(\cdot|e_t)$ under $g_{e_t}$:
\begin{equation}
    \gH(g_{e_t}(\alpha_t) | f(e_t)) = \gH(\overline{q}(\cdot | \overline{e}_t))
\end{equation}
The second term is the expected value of the conditional entropy of selecting a concrete action $\alpha_t$ given that its abstraction is $g_{e_t}(\alpha_t)$.

\textbf{Combining Terms:} Substituting the transformed reward and decomposed entropy back into the expression for $\gL(q^\uparrow)$:
\begin{align}
    \gL(q^\uparrow) &= \E_{\overline{q}(\overline{\tau})} \left[ \sum_{t=1}^T \left( \overline{R}(\overline{e}_t, \overline{\alpha}_t) + \overline{f}(\overline{e}_t, \overline{\alpha}_t) + \gH(\overline{q}(\cdot | \overline{e}_t)) \right) \right] \nonumber \\
    & \quad + \E_{q^\uparrow(\tau)} \left[ \sum_{t=1}^T \E_{q^\uparrow(\alpha_t|e_t)}[\gH(q^\uparrow(\alpha_t | e_t, g_{e_t}(\alpha_t)))] \right] \\
    &= \overline{\gL}(\overline{q}) + \E_{q^\uparrow(\tau)} \left[ \sum_{t=1}^T \gH(q^\uparrow(\alpha_t | e_t, g_{e_t}(\alpha_t))) \right] \\
    & \ge \overline{\gL}(\overline{q})
\end{align}
This completes the proof.
\end{proof}

This theorem reveals that the ELBO in the original MDP decomposes
into two parts: (1) the ELBO of the abstract MDP, and (2) the
expected conditional entropy over concrete actions within an
abstract action's equivalence class. Since entropy is
non-negative, $\gL(q^\uparrow) \geq
\overline{\gL}(\overline{q})$.

Crucially, the conditional entropy term depends only on how the
policy $q^\uparrow$ distributes probability mass among the
concrete actions that map to the same abstract action via
$g_{e_t}$. The optimization of the abstract policy $\overline{q}$
via VMOC directly maximizes the $\overline{\gL}(\overline{q})$
term. This shows that improving the policy in the abstract space
guarantees an improvement in the variational objective of the
original, more complex problem. The value equivalence results
from Section 4.1 further guarantee that the optimal solution is
not lost in this abstraction. Therefore, performing variational
inference with VMOC on the abstract MDP is a valid and principled
approach to solving the original control problem.

\section{Experiments}

\subsection{Efficiency of VMOC on Locomotion Tasks}

In this section, we design experiments on the challenging single
task OpenAI Gym MuJoCo~\cite{brockman2016openai} environments (10
environments) to test Variational Markovian Option Critic
(VMOC)'s performance over other option variants and non-option
baselines.

For VMOC in all environments, we fix the temperature rate for
both $\alpha^O$ and $\alpha^A$ to $0.05$; we add an exploration
noise $\gN(\mu=0,\sigma=0.2)$ during exploration. For all
baselines, we follow DAC~\cite{zhang2019dac}'s open source
implementations and compare our algorithm with six baselines,
five of which are option variants, \textit{i.e.},
MOPG~\cite{li2022hit}, DAC+PPO, AHP+PPO \cite{levy2011unified},
IOPG \cite{smith2018inference}, PPOC
\cite{klissarov2017learnings}, OC \cite{bacon2017option} and PPO
\cite{schulman2017proximal}. All baselines' parameters used by
DAC remain unchanged over 1 million environment steps to
converge. Figures are plotted following DAC's style: curves are
averaged over 10 independent runs and smoothed by a sliding
window of size 20. Shaded regions indicate standard deviations.
All experiments are run on an Intel® Core™ i9-9900X CPU @ 3.50GHz
with a single thread and process. Our implementation details are
summarized in Appendix~\ref{sec:append_imp}. For a fair
comparison, we follow option literature conventions and use four
options in all implementations. Our code is available in
supplemental materials.

We evaluate the performance of VMOC against six option-based
baselines (MOPG~\cite{li2022hit}, DAC+PPO~\cite{zhang2019dac},
AHP+PPO~\cite{levy2011unified}, IOPG~\cite{smith2018inference},
PPOC~\cite{klissarov2017learnings}, and
OC~\cite{bacon2017option}) as well as the hierarchy-free PPO
algorithm~\cite{schulman2017proximal}. Previous
studies~\cite{klissarov2017learnings,smith2018inference,harb2018waiting,zhang2019dac}
have suggested that option-based algorithms do not exhibit
significant advantages over hierarchy-free algorithms in
single-task environments. Nonetheless, our results demonstrate
that VMOC significantly outperforms all baselines in terms of
episodic return, convergence speed, step variance, and variance
across 10 runs, as illustrated in Figure~\ref{fig:all_exp}. The
only exception is the relatively simple InvertedDoublePendulum
environment, which we suspect is due to hyper-parameter tuning
issues and will be addressed in future work.
\begin{figure}[H]
  \centering
  \includegraphics[width=1\linewidth]{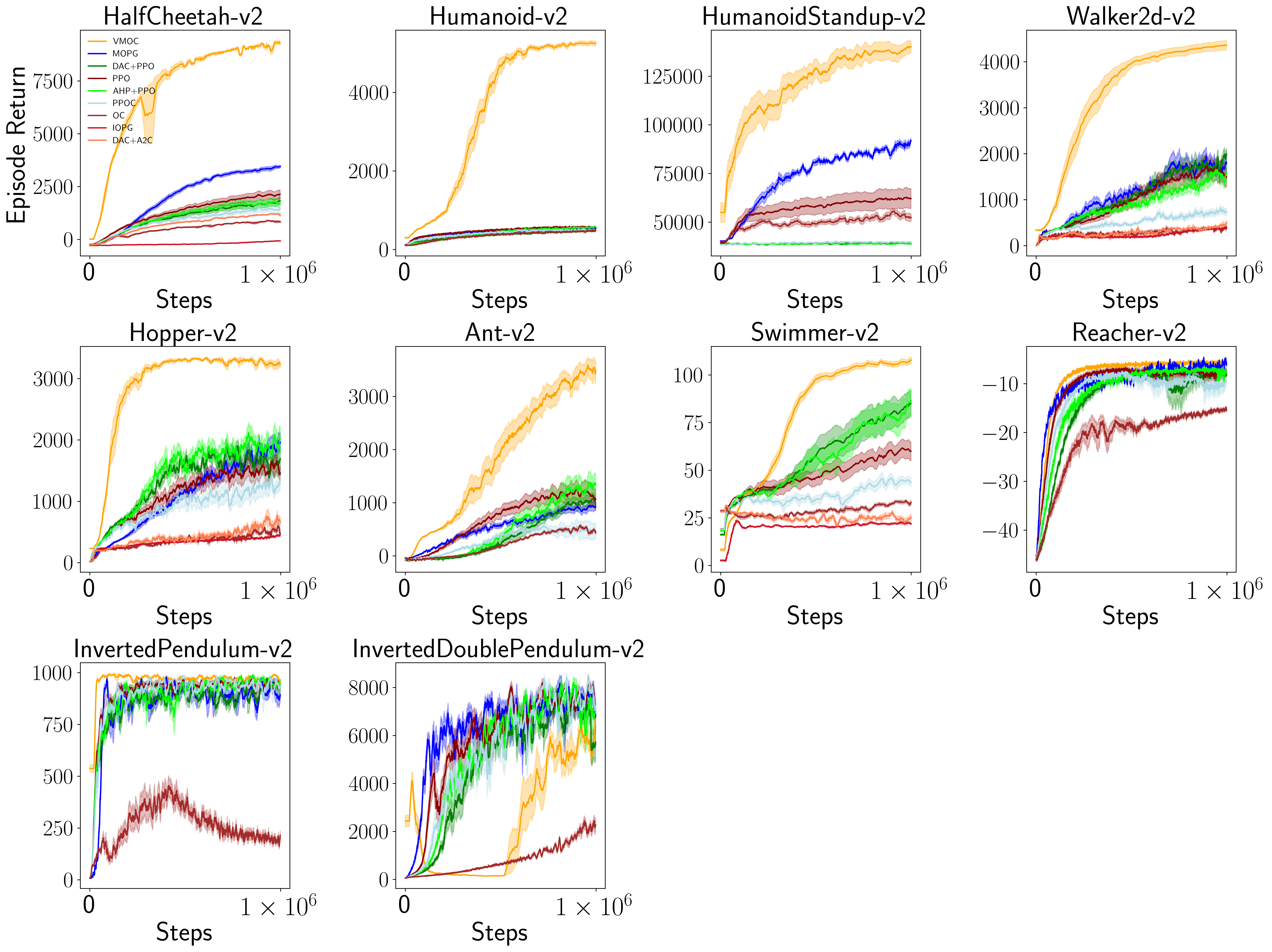}
  \caption{\label{fig:all_exp} Experiments on Mujoco
    Environments. Curves are averaged over 10 independent runs
    with different random seeds and smoothed by a sliding window
    of size 20. Shaded regions indicate standard deviations.}
\end{figure}

Notably, VMOC exhibits superior performance on the Humanoid-v2
and HumanoidStandup-v2 environments. These environments are
characterized by a large state space ($\mathcal{S} \in
\mathbb{R}^{376}$) and action space ($\mathcal{A} \in
\mathbb{R}^{17}$), whereas other environments typically have
state dimensions less than 20 and action dimensions less than 5.
The enhanced performance of VMOC in these environments can be
attributed to its maximum entropy capability: in large
state-action spaces, the agent must maximize rewards while
exploring a diverse set of state-action pairs. Maximum likelihood
methods tend to quickly saturate with early rewarding
observations, leading to the selection of low-entropy options
that converge to local optima.

A particularly relevant comparison is with the Markovian Option
Policy Gradient (MOPG)~\cite{li2022hit}, as both VMOC and MOPG
are developed based on HiT-MDPs and employ option embeddings.
Despite being derived under the maximum entropy framework, MOPG
utilizes an on-policy gradient descent approach. Our experimental
results show that VMOC's performance surpasses that of MOPG,
highlighting the limitations of on-policy methods, which suffer
from shortsighted rollout lengths and quickly saturate to early
high-reward observations. In contrast, VMOC's variational
off-policy approach effectively utilizes the maximum entropy
framework by ensuring better exploration and stability across the
learning process. Additionally, the off-policy nature of VMOC
allows it to reuse samples from a replay buffer, enhancing sample
efficiency and promoting greater diversity in the learned
policies. This capability leads to more robust learning, as the
algorithm can leverage a broader range of experiences to improve
policy optimization.

\subsection{Cold-start VMOC on Language Model Tasks}

To evaluate the effectiveness of initializing option embeddings via supervised fine-tuning, we apply our cold-start procedure (VMOC-SFT) to a suite of mathematical and logical reasoning tasks. We assess its performance on both in-distribution and out-of-distribution (OOD) datasets to measure its reasoning capabilities and generalization robustness.

\subsubsection{Datasets}
Our evaluation leverages six public datasets, split into in-domain and out-of-domain categories for a comprehensive analysis.

\paragraph{In-Domain Datasets.} These datasets are used to evaluate the model's primary reasoning abilities after the cold-start phase.
\begin{itemize}
    \item \textbf{GSM8k-Aug} \cite{deng2023implicit}: A large-scale mathematical reasoning dataset derived from GSM8K \cite{cobbe2021training}. It contains 385K training examples generated by GPT-4, with reasoning steps presented as structured mathematical expressions.
    \item \textbf{GSM8k-Aug-NL}~\cite{deng2023implicit}: A variant of GSM8k-Aug where the intermediate reasoning steps are expressed in natural language, testing the model's ability to process less structured thought processes.
    \item \textbf{CommonSense}~\cite{talmor2018commonsenseqa}: This logical reasoning dataset requires the model to perform multi-hop reasoning over a set of conditions to verify a statement, simulating a simplified theorem-proving environment.
\end{itemize}

\paragraph{Out-of-Domain Datasets.} To test the generalization of the learned reasoning strategies, we evaluate the model trained on GSM8k-Aug on three OOD datasets without further fine-tuning.
\begin{itemize}
    \item \textbf{SVAMP} \cite{patel-etal-2021-nlp}: A collection of elementary-level math word problems that require understanding short narratives.
    \item \textbf{GSM-HARD} \cite{gao2023pal}: A more challenging version of the GSM8K test set, featuring an increased range of numerical values to test for robustness.
    \item \textbf{MultiArith} \cite{roy-roth-2015-solving}: A dataset composed of multi-step arithmetic word problems, requiring the model to sequence multiple calculations correctly.
\end{itemize}

\subsubsection{Experimental Setup}
We implement our cold-start training on the LLaMA3.2-1b
\cite{grattafiori2024llama}. The training process utilizes the
AdamW optimizer with a learning rate of 5e-5, a 10\% warm-up
phase, and subsequent linear decay. For the VMOC-SFT specific
hyperparameters, we use a latent space with 6 distinct option
embeddings. The KL divergence weight, $\beta$, is set to 0.1 to
balance the reconstruction loss with the regularization of the
latent space. All experiments were conducted on an Amazon EC2
ml.p5d.24xlarge instance, equipped with 8 NVIDIA H200 GPUs.

\subsubsection{Baselines}
We compare VMOC-SFT against several strong baselines that represent different approaches to incorporating reasoning in language models.
\begin{itemize}
    \item \textbf{CoT-SFT}: A standard supervised fine-tuning approach where the model is explicitly trained on prompts, chain-of-thought reasoning, and final answers.
    \item \textbf{No-CoT-SFT}: The model is fine-tuned only on prompt-answer pairs, without exposure to intermediate reasoning steps.
    \item \textbf{Pause-CoT-SFT} \cite{goyal2023think}: A variant of No-CoT-SFT where special `<pause>` tokens are inserted between the prompt and answer during training, hypothesized to give the model additional computational capacity for implicit reasoning.
    \item \textbf{iCoT} \cite{deng2024explicit}: An "internalization" strategy where the explicit CoT is gradually removed during training, forcing the model to reason implicitly.
    \item \textbf{COCONUT} \cite{coconut}: A method that also internalizes reasoning, but instead of deleting the CoT, it replaces it with latent codes for autoregressive generation.
    \item \textbf{CODI} \cite{codi}: A distillation-based approach that jointly trains an explicit CoT task and an implicit reasoning task, aligning hidden states to transfer reasoning abilities into a continuous latent space.
\end{itemize}

\subsubsection{Results and Analysis}

The performance of VMOC-SFT against all baselines is presented in Table \ref{tab:model_comparison}. Our method demonstrates a distinct performance profile, excelling in tasks that require abstract reasoning and robustness to increased difficulty.

\begin{table}[h!]
\centering
\caption{Comparison of different models on In-Sample and Out-of-Distribution (OOD) datasets. OOD models were trained on GSM8k-NL. Best overall results are in bold, while best results for our method are in bold red.}
\label{tab:model_comparison}
\resizebox{\textwidth}{!}{%
\begin{tabular}{l ccc ccc}
\toprule
& \multicolumn{3}{c}{\textbf{In Sample}} & \multicolumn{3}{c}{\textbf{OOD (Trained On GSM8k-NL)}} \\
\cmidrule(lr){2-4} \cmidrule(lr){5-7}
\textbf{Model} & \textbf{GSM8k-Aug} & \textbf{GSM8k-Aug-NL} & \textbf{CommonSense} & \textbf{SVAMP} & \textbf{GSM-HARD} & \textbf{MultiArith} \\
\midrule
CoT-SFT         & 60.8 & 55.1 & 68.2 & 66.7 & 15.8 & \textbf{99.3} \\
CODI            & \textbf{55.6} & \textbf{49.7} & 74.0 & \textbf{61.1} & 12.8 & 96.1 \\
\textcolor{red}{VMOC-SFT} & 37.5 & 43.8 & \textcolor{red}{\textbf{78.1}} & 31.3 & \textcolor{red}{\textbf{15.6}} & 62.5 \\
Coconut         & 45.3 & 27.2 & 60.6 & 48.8 & 9.9  & 90.1 \\
No-CoT-SFT      & 28.7 & 28.7 & 74.9 & 44.1 & 6.8  & 69.2 \\
Pause-CoT-SFT   & 28.1 & 28.1 & -    & 41.2 & 6.7  & 65.3 \\
iCoT            & 19.2 & 15.1 & 72.6 & 40.9 & 5.2  & 45.8 \\
\bottomrule
\end{tabular}
}
\end{table}

On the in-sample datasets, VMOC-SFT's performance on the GSM8k
variants is lower than that of explicit CoT-based methods like
CoT-SFT and CODI. This is expected, as our variational approach
optimizes for a compressed latent representation of reasoning
rather than direct imitation of the reasoning text. However, this
trade-off proves highly beneficial on the \textbf{CommonSense}
dataset, where VMOC-SFT achieves a state-of-the-art score of
\textbf{78.1}. This suggests that our method's ability to learn a
discrete latent space of reasoning "options" is particularly
effective for capturing the abstract, multi-hop logical
structures in this task, which are less amenable to simple
textual imitation.

This advantage in handling complexity is further validated on the
OOD benchmarks. While VMOC-SFT does not top the charts on SVAMP
or MultiArith, it achieves the best score on \textbf{GSM-HARD}
with \textbf{15.6}. This result is significant, as it indicates
that the reasoning strategies encoded in our latent options are
more robust and generalize better to problems of increased
difficulty within the same domain. While other models may overfit
to the patterns in the training data, our approach learns a more
fundamental set of reasoning primitives, allowing it to adapt
more effectively when faced with harder problems. This highlights
the value of the cold-start procedure in equipping the model with
a resilient and abstract reasoning framework before it is further
refined through reinforcement learning.

\section{Related Work}

The VMOC incorporates three key ingredients: the option
framework, a structural variational inference based off-policy
algorithm and latent variable policies. We review prior works
that draw on some of these ideas in this section. The options
framework~\cite{sutton1999between} offers a promising approach
for discovering and reusing temporal abstractions, with options
representing temporally abstract skills. Conventional option
frameworks~\cite{precup2000temporal}, typically developed under
the maximum likelihood (MLE) framework with few constraints on
options behavior, often suffer from the option degradation
problem~\cite{levy2011unified,bacon2017option}. This problem
occurs when options quickly saturate with early rewarding
observations, causing a single option to dominate the entire
policy, or when options switch every timestep, maximizing policy
at the expense of skill reuse across tasks. On-policy option
learning algorithms~\cite{bacon2017option, bacon2018temporal,
  zhang2019dac, li2020skill, li2022hit} aim to maximize expected
return by adjusting policy parameters to increase the likelihood
of high-reward option trajectories, which often leads to focusing
on low-entropy options. Several techniques~\cite{harb2018waiting,
  harutyunyan2019termination, kamat2020diversity} have been
proposed to enhance on-policy algorithms with entropy-like
extrinsic rewards as regularizers, but these often result in
biased optimal trajectories. In contrast, the maximum entropy
term in VMOC arises naturally within the variational framework
and provably converges to the optimal trajectory.

Although several off-policy option learning algorithms have been
proposed~\cite{daniel2016probabilistic,shiarlis2018taco,smith2018inference,
  wulfmeier2020data}, these typically focus on improving sample
efficiency by leveraging the control as inference framework.
Recent works~\cite{smith2018inference} aim to enhance sample
efficiency by inferring and marginalizing over options, allowing
all options to be learned simultaneously.
\citename{wulfmeier2020data} propose off-policy learning of all
options across every experience in hindsight, further boosting
sample efficiency. However, these approaches generally lack
constraints on options behavior. A closely related
work~\cite{li2020soac} also derives a variational approach under
the option framework; however, it is based on probabilistic
graphical model that we believe are incorrect, potentially
leading to convergence issues. Additionally, our algorithm
enables learning options as latent embeddings, a feature not
present in their approach.

Recently, several studies have extended the maximum entropy
reinforcement learning framework to discover skills by
incorporating additional latent variables. One class of
methods~\cite{hausman2018learning,gupta2019relay} maintains
latent variables constant over the duration of an episode,
providing a time-correlated exploration signal. Other
works~\cite{haarnoja2018soft, zhang2022latent} focus on
discovering multi-level action abstractions that are suitable for
repurposing by promoting skill distinguishability, but they do
not incorporate temporal abstractions. Studies such
as~\cite{pertsch2020long, ajay2020opal, co2018self} aim to
discover temporally abstract skills essential for exploration,
but they predefine their temporal resolution. In contrast, VMOC
learns temporal abstractions as embeddings in an end-to-end
data-driven approach with minimal prior knowledge encoded in the
framework.

State abstraction is a cornerstone of scalable reinforcement
learning. A formal and powerful tool for this is the MDP
homomorphism, which defines an abstraction map between a large,
complex MDP and a smaller, simpler one while preserving key
structural properties. Foundational work established these
concepts for discrete state and action spaces. More recently,
\citet{panangaden2024policy} extended the theory to standard
continuous MDPs, providing guarantees for optimal value
equivalence and policy lifting under much broader conditions .
Our work builds directly on this, making a novel contribution by
extending the continuous homomorphism framework to HiT-MDPs. By
employing the mathematical structure of vector bundles , we
formally define abstractions over the joint state-option space,
proving that optimality is preserved. This provides a rigorous
theoretical justification for learning in abstract hierarchical
models like VMOC.

hain-of-Thought (CoT) prompting dramatically improved the
ability of LLMs to solve complex reasoning tasks by generating
explicit intermediate steps~\cite{wei2022chain}. However, the
computational overhead of generating these steps has spurred
research into more efficient methods . A key direction is
``latent reasoning,'' where the intermediate computations are
performed within the model's hidden states rather than being
decoded into text . Some approaches aim to achieve this
implicitly by modifying the training process, for instance by
removing the explicit CoT during
fine-tuning~\cite{deng2024explicit} or inserting special tokens
to encourage latent computation~\cite{goyal2023think}. Other
works focus on explicitly engineering a latent workspace, for
instance by distilling explicit CoT into continuous hidden
states~\cite{shen2025codi} or learning to replace reasoning text
with latent codes~\cite{coconut}. Our work contributes a new
perspective by proposing the discrete option space of the HiT-MDP
framework as a structured medium for latent reasoning. By using a
variational objective during a cold-start SFT phase, we learn to
map explicit reasoning demonstrations to a vocabulary of discrete
option embeddings, creating a model that ``thinks'' by composing
these learned latent primitives.

\section{Conclusion}

In this paper, we introduced the Variational Markovian Option
Critic (VMOC), a novel off-policy algorithm that addresses key
challenges in hierarchical reinforcement learning. By integrating
a variational inference framework on option-induced HiT-MDPs,
VMOC naturally promotes the discovery of diverse and effective
options through maximum entropy intrinsic rewards. To ground our
framework theoretically, we extended the theory of continuous MDP
homomorphisms to the HiT-MDP setting, proving that learning in a
homomorphically abstract option space preserves optimal value
functions and provides a principled path to scalable HRL.
Furthermore, we demonstrated the versatility of our approach by
proposing a cold-start training procedure for Large Language
Models. This method distills explicit Chain-of-Thought
demonstrations into the latent option space, enabling efficient
and implicit reasoning. Extensive experiments show that VMOC
significantly outperforms strong baselines in complex locomotion
tasks and achieves excellent results on logical reasoning
benchmarks, validating its effectiveness and scalability for both
control and language.

\bibliographystyle{icml2021}

\appendix
\section{Proofs}

\subsection{\autoref{thm:var_soft_mdp}}
\label{app:thm_var_soft_mdp}

\textbf{\autoref{thm:var_soft_mdp}} (Convergence Theorem for
Structured Variational Policy Iteration). \textit{ Let \(\tau\)
  be the latent variable and \(\gE\) be the observed variable.
  Define the variational distribution \(q(\tau)\) and the
  log-likelihood \(\log P(\gE)\). Let \( M: q^{[k]} \rightarrow
  q^{[k+1]} \) represent the mapping defined by the EM steps
  inference update, so that \( q^{[k+1]} = M(q^{[k]}) \). The
  likelihood function increases at each iteration of the
  variational inference algorithm until the conditions for
  equality are satisfied and a fixed point of the iteration is
  reached:
  \[
    \log P(\gE \mid q^{[k+1]}) \geq \log P(\gE \mid q^{[k]}), \text{ with equality if and only if}
  \]
  \[
    \mathcal{L}(q^{[k+1]}, P) = \mathcal{L}(q^{[k]}, P)
  \]
  and
  \[
    \text{D}_{\text{KL}}(q^{[k+1]}(\tau) \parallel P(\tau \mid
    \gE)) = \text{D}_{\text{KL}}(q^{[k]}(\tau) \parallel P(\tau
    \mid \gE)).
  \]
}

\begin{proof}

  Let \(\tau\) be the latent variable and \(\gE\) be the observed
  variable. Define the evidence lower bound (ELBO) as
  \(\mathcal{L}(q, P)\) and the Kullback-Leibler divergence as
  \(\text{D}_{\text{KL}}(q \parallel P)\), where \(q(\tau)\)
  approximates the posterior distribution and \(P(\gE \mid
  \tau)\) is the likelihood.

  The log-likelihood function \( \log P(\gE) \) can be decomposed
  as:
\[
  \log P(\gE) = \mathcal{L}(q, P) + \text{D}_{\text{KL}}(q(\tau)
  \parallel P(\tau \mid \gE)),
\]
where
\[
\mathcal{L}(q, P) = \mathbb{E}_{q(\tau)} \left[ \log P(\gE, \tau) - \log q(\tau) \right]
\]
and
\[
  \text{D}_{\text{KL}}(q(\tau) \parallel P(\tau \mid \gE)) =
  \mathbb{E}_{q(\tau)} \left[ \log \frac{q(\tau)}{P(\tau \mid
      \gE)} \right].
\]

Let \( M: q^{[k]} \rightarrow q^{[k+1]} \) represent the mapping
defined by the variational inference update, so that \( q^{[k+1]}
= M(q^{[k]}) \). If \(q^*\) is a variational distribution that
maximizes the ELBO, so that \(\log P(\gE \mid q^*) \geq \log
P(\gE \mid q)\) for all \(q\), then \(\log P(\gE \mid M(q^*)) =
\log P(\gE \mid q^*)\). In other words, the maximizing
distributions are fixed points of the variational inference
algorithm. Since the likelihood function is bounded (for
distributions of practical interest), the sequence of variational
distributions \( q^{[0]}, q^{[1]}, \ldots, q^{[k]} \) yields a
bounded nondecreasing sequence \(\log P(\gE \mid q^{[0]}) \leq
\log P(\gE \mid q^{[1]}) \leq \cdots \leq \log P(\gE \mid
q^{[k]}) \leq \log P(\gE \mid q^{[k]})\) which must converge as
\( k \rightarrow \infty \).

\end{proof}

\subsection{\autoref{thm:soft_mopi}}
\label{app:thm_soft_mopi}

\textbf{\autoref{thm:soft_mopi}} (Soft Option Policy Iteration
Theorem). \textit{ Repeated optimizing $\gL$ and $\KL$ defined in
  \Eqref{eq:lme_hitmdp} from any \(\pi^A_0,\pi^O_0 \in \Pi\)
  converges to optimal policies \(\pi^{A*}, \pi^{O*}\) such that
  $Q^{soft*}_O[\rvs_t,\rvo_t] \geq Q^{soft}_O[\rvs_t,\rvo_t]$ and
  $Q^{soft*}_A[ \rvs_t,\rvo_t,\rva_t]\geq Q^{soft}_A[
  \rvs_t,\rvo_t,\rva_t]$, for all \(\pi^A_0,\pi^O_0 \in \Pi\) and
  \((\rvs_t, \rva_t, \rvo_t) \in \mathcal{S} \times \mathcal{A}
  \times \gO \), assuming $|\gS|<\infty, \; |\gO|<\infty, \;
  |\gA|<\infty$. }

\begin{proof}
  Define the entropy augmented reward as
  $r^{soft}(\rvs_t,\rva_t) = r(\rvs_t,\rva_t)+\gH[\pi^A]$ and
  \(f^{soft}(\rvo_t,\rvs_t,\rva_t,\rvo_{t-1}) =
  f(\rvo_t,\rvs_t,\rva_t,\rvo_{t-1}) + \gH[\pi^O]\) and rewrite
  Bellman Backup functions as,
  \begin{align*}
    Q_O[\rvs_t,\rvo_t]&= f^{soft}(\cdot) + \E_{\rva_t \sim \pi^A}\left[Q_A[\rvs_t,\rvo_t,\rva_t]\right],\\
    Q_A[ \rvs_t,\rvo_t,\rva_t]&= r^{soft}(s,a) + \E_{\rvs_{t+1}\sim P(\rvs_{t+1}|\rvs_t,\rva_t)}\left[\E_{\rvo_{t+1}\sim\pi^O}\left[ Q_O[\rvs_{t+1},\rvo_{t+1}]\right]\right]
  \end{align*}

  We start with proving the convergence of soft option policy
  evaluation. As with the standard Q-function and value function,
  we can relate the Q-function at a future state via a \emph{
    Bellman Operator} $\mathcal{T}^{soft}$. The option-action
  value function satisfies the Bellman Operator
  $\mathcal{T}^{soft}$
    \begin{align*}
      \label{eq:sa_q_a}
      \mathcal{T}^{soft}Q_A[ \rvs_t,\rvo_t,\rva_t]&=\E[G_t| \rvs_t,\rvo_t,\rva_t]\nonumber\\
                                                   &= r^{soft}(s,a) + \gamma\sum_{\rvs_{t+1}}P(\rvs_{t+1}|\rvs_t,\rva_t)Q_O[\rvs_{t+1},\rvo_t],
    \end{align*}

    As with the standard convergence results for policy
    evaluation~\cite{sutton2018reinforcement}, by the definition
    of $\mathcal{T}^{soft}$ (\Eqref{eq:sa_q_a}) the option-action
    value function $Q_A^{\pi_A}$ is a fixed point.

    To prove the $\mathcal{T}^{soft}$ is a contraction, define a norm
    on $V$-values functions $V$ and $U$
    \begin{align}
      \|V - U \|_{\infty } \triangleq \max_{\bar{s} \in \bar{S}}|V(\bar{s}) - U(\bar{s}) |.
    \end{align}
    where $\bar{s}=\{s,o\}$.

    By recurssively apply the Hidden Temporal Bellman Operator
    $\mathcal{T}^{soft}$, we have:

    \begin{align}
      \label{eq:app_v_bell}
      Q_O[\rvs_t,\rvo_{t-1}]&=\E[G_t|\rvs_t,\rvo_{t-1}]= \sum_{\rvo_t}P(\rvo_t|\rvs_t,\rvo_{t-1})Q_O[ \rvs_t,\rvo_t]\nonumber\\
                                      &=\sum_{\rvo_t}P(\rvo_t|\rvs_t,\rvo_{t-1})\sum_{\rva_t}P(\rva_t|\rvs_t,\rvo_{t})\bigg[r(s,a) + \gamma\sum_{\rvs_{t+1}}P(\rvs_{t+1}|\rvs_t,\rva_t)Q_O[\rvs_{t+1},\rvo_t]\bigg]\nonumber\\
                                      &=r(s,a) + \gamma\sum_{\rvo_t}P(\rvo_t|\rvs_t,\rvo_{t-1})\sum_{\rva_t}P(\rva_t|\rvs_t,\rvo_{t})\sum_{\rvs_{t+1}}P(\rvs_{t+1}|\rvs_t,\rva_t)Q_O[\rvs_{t+1},\rvo_t]\nonumber\\
                                      &=r(s,a) + \gamma\sum_{\rvo_t,\rvs_{t+1}}P(\rvs_{t+1},\rvo_t|\rvs_t,\rvo_{t-1})Q_O[\rvs_{t+1},\rvo_t]\nonumber\\
                                      &=r(s,a) + \gamma E_{\rvs_{t+1},\rvo_t}\bigg[Q_O[\rvs_{t+1},\rvo_t]\bigg]
    \end{align}

    Therefore, by applying \Eqref{eq:app_v_bell} to $V$ and $U$
    we have:
    \begin{flalign}
  \|&T^{\pi}V-T^{\pi} U\|_{\infty}\nonumber\\
  &= \max_{\bar{s} \in \bar{S}} \bigg| \gamma
  E_{\rvs_{t+1},\rvo_t}\bigg[Q_O[\rvs_{t+1},\rvo_t]\bigg]
  - \gamma
  E_{\rvs_{t+1},\rvo_t}\bigg[U[\rvs_{t+1},\rvo_t]\bigg]\bigg|
  \nonumber\\
  & =\gamma \max_{\bar{s} \in \bar{S}}
  E_{\rvs_{t+1},\rvo_t}\bigg[\bigg|Q_O[\rvs_{t+1},\rvo_t]-U[\rvs_{t+1},\rvo_t]\bigg|\bigg]\nonumber\\
  & \leq \gamma \max_{\bar{s} \in \bar{S}}
  E_{\rvs_{t+1},\rvo_t}\bigg[\gamma \max_{\bar{s} \in \bar{S}}\bigg|Q_O[\rvs_{t+1},\rvo_t]-U[\rvs_{t+1},\rvo_t]\bigg|\bigg]\nonumber\\
  & \leq \gamma \max_{\bar{s} \in \bar{S}}|V[\bar{s}]-U[\bar{s}]|
  \nonumber\\
  &= \gamma \| V - U \|_{_{\infty}}
\end{flalign}
Therefore, $\mathcal{T}^{soft}$ is a contraction. By the fixed
point theorem, assuming that throughout our computation the
$Q_A[\cdot,\cdot]$ and $Q_O[\cdot]$ are bounded and $\sA<\infty$,
the sequence $Q_A^k$ defined by
$Q_A^{k+1}=\mathcal{T}^{soft}Q_A^k$ will converge to the
option-action value function $Q_A^{\pi_A}$ as $k\rightarrow
\infty$.

The convergence results of and the Soft Option Policy Improvement
Theorem then follows conventional Soft Policy Improvement
Theorem~\autoref{thm:var_soft_mdp}. Consequently, the Soft Option
Policy Iteration Theorem follows directly from these results.

\end{proof}

\subsection{Derivation of \Eqref{eq:lme_hitmdp}}
\label{app:deriv_lme_hitmdp}

\begin{align*}
  \gL(q(\tau),P(\tau,\gE^A_{1:T},\gE^O_{1:T})) &= \E_{q(\tau)}[\log P(\tau,\gE^A_{1:T},\gE^O_{1:T}) - \log q(\tau)]\\
                                               &= \E_{q(\tau)}[\log P(\tau|\gE^A_{1:T},\gE^O_{1:T}) + \log P(\gE^A_{1:T},\gE^O_{1:T}) - \log q(\tau)]\\
                                               &= \E_{q(\tau)}[\log P(\tau|\gE^A_{1:T},\gE^O_{1:T})  - \log q(\tau)] + \E_{q(\tau)}\log P(\gE^A_{1:T},\gE^O_{1:T})\\
                                               &= \E_{q(\tau)}[\frac{\log P(\tau|\gE^A_{1:T},\gE^O_{1:T})}{\log q(\tau)}] + \log P(\gE^A_{1:T},\gE^O_{1:T})\\
                                               &= -\KL(\log q(\tau)\parallel\log P(\tau|\gE^A_{1:T},\gE^O_{1:T})) + \log P(\gE^A_{1:T},\gE^O_{1:T})
\end{align*}

\subsection{\autoref{thm:var_hitmdp}}
\label{app:thm_var_hitmdp}

\textbf{\autoref{thm:var_hitmdp}} (Convergence Theorem for
Variational Markovian Option Policy Iteration). \textit{
  Let \(\tau\) be the latent variable and \(\gE^A, \gE^O\) be the
  ground-truth optimality variables. Define the variational
  distribution \(q(\tau)\) and the true log-likelihood of
  optimality \(\log P(\gE^A, \gE^O)\). iterates according to the
  update rule $q^{k+1}=\argmax_q
  \gL(q(\tau),P(\tau,\gE^A_{1:T},\gE^O_{1:T}))$ converges to the
  maximum value bounded by the data log-likelihood.
}

\begin{proof}
  The objective is to maximize the ELBO with respect to the
  policy \( q \). Formally, this can be written as:
  \[
  q^{k+1} = \arg \max_q \mathcal{L}(q, P).
  \]

  Suppose we $q$ is a neural network function approximator,
  assuming the continuity and differentiability of \( q \) with
  respect to its parameters. Using stochastic gradient descent
  (SGD) to optimize the parameters guarantees that the ELBO
  increases, such that \( \mathcal{L}(q^{k+1}, P) \geq
  \mathcal{L}(q^k, P) \).

  Rearranging \Eqref{eq:lme_hitmdp} we get:
  \begin{align*}
    \KL(q^{k+1}(\tau)||P(\tau|\gE^A_{1:T},\gE^O_{1:T})) & =-L(q^{k+1}(\tau),P(\tau,\gE^A_{1:T},\gE^O_{1:T})) + \log P(\gE^A_{1:T},\gE^O_{1:T}) \\
    &\leq -L(q^k(\tau),P(\tau,\gE^A_{1:T},\gE^O_{1:T})) + \log P(\gE^A_{1:T},\gE^O_{1:T}) \\
    &= \KL(q^{k}(\tau)||P(\tau|\gE^A_{1:T},\gE^O_{1:T}))
  \end{align*}
  Thus, each SGD update not only potentially increases the ELBO
  but also decreases the KL divergence, moving \( q \) closer to
  \( P \). Given the properties of SGD and assuming appropriate
  learning rates and sufficiently expressive neural network
  architectures, the sequence \( \{ q^k \} \) converges to a
  policy \( q^* \) that minimizes the KL divergence to the true
  posterior.
\end{proof}

\section{VMOC Algorithm}
\label{app:algo}

\begin{algorithm}
\caption{VMOC Algorithm}
\begin{algorithmic}[1]
\STATE Initialize parameter vectors $\psi^A$, $\psi^O$, $\theta^O$, $\theta^A$
\FOR{each epoch}
\STATE Collect trajectories $\{\rvo_{t-1}, \rvs_t, \rva_t,\rvo_t\}$ into the replay buffer
    \FOR{each gradient step}
        \STATE Update $Q_{\psi_i^A}^{soft}$: $\psi_i^A \leftarrow
        \psi_i^A - \eta_{Q^A} \nabla
        J_{Q_{\psi_i^A}^{soft}}\;\text{for} \; i\in \{1,2\}$
        \STATE Update $Q_{\psi_i^O}^{soft}$: $\psi_i^O \leftarrow \psi_i^O - \eta_{Q^O} \nabla J_{Q_{\psi_i^O}^{soft}}\;\text{for} \; i\in \{1,2\}$
        \STATE Update $\pi_{\theta^O}^O$: $\theta^O \leftarrow \theta^O - \eta_{\pi^O} \nabla J_{\pi^O}$
        \STATE Update $\pi_{\theta^A}^A$: $\theta^A \leftarrow \theta^A - \eta_{\pi^A} \nabla J_{\pi^A}$
        \STATE Update target networks: $\bar{\psi}^A \leftarrow \sigma \psi^A + (1 - \sigma) \bar{\psi}^A$, $\bar{\psi}^O \leftarrow \sigma \psi^O + (1 - \sigma) \bar{\psi}^O$
        \STATE Update temperature factors: $\alpha^O \leftarrow \alpha^O - \eta_{\alpha^O} \nabla J_{\alpha^O}$, $\alpha^A \leftarrow \alpha^A - \eta_{\alpha^A} \nabla J_{\alpha^A}$
    \ENDFOR
\ENDFOR
\end{algorithmic}
\end{algorithm}

\section{Implementation Details}
\label{sec:append_imp}

\subsection{Hyperparameters}
\label{sec:append_hyperparams}

In this section we summarize our implementation details. For a
fair comparison, all baselines: MOPG~\cite{li2022hit}, DAC+PPO
\cite{zhang2019dac}, AHP+PPO \cite{levy2011unified}, PPOC
\cite{klissarov2017learnings}, OC \cite{bacon2017option} and PPO
\cite{schulman2017proximal} are from DAC's open source Github
repo: \url{https://github.com/ShangtongZhang/DeepRL/tree/DAC}.
Hyper-parameters used in DAC~\cite{zhang2019dac} for all these
baselines are kept unchanged.

\textbf{VMOC Network Architecture:} We use Pytorch to build
neural networks. Specifically, for option embeddings, we use an
embedding matrix $\mW_S\in \sR^{4\times 40}$ which has $4$
options ($4$ rows) and an embedding size of $40$ ($40$ columns).
For layer normalization we use Pytorch's built-in function
LayerNorm
\footnote{https://pytorch.org/docs/stable/generated/torch.nn.LayerNorm.html}.
For Feed Forward Networks (FNN), we use a 2 layer FNN with ReLu
function as activation function with input size of state-size,
hidden size of $[256,256]$, and output size of action-dim
neurons. For Linear layer, we use built-in Linear
function\footnote{https://pytorch.org/docs/stable/generated/torch.nn.Linear.html}
to map FFN's outputs to $4$ dimension. Each dimension acts like a
logit for each skill and is used as density in Categorical
distribution\footnote{https://github.com/pytorch/pytorch/blob/master/torch/distributions/categorical.py}.
For both action policy and critic module, FFNs are of the same
size as the one used in the skill policy.

\textbf{Preprocessing:} States are normalized by a running
estimation of mean and std.

\textbf{Hyperparameters for all on-policy option variants:} For a
fair comparison, we use exactly the same parameters of PPO as DAC
. Specifically:
\begin{itemize}
\item Optimizer: Adam with $\epsilon= 10^{-5}$ and an initial
  learning rate $3 \times 10^{-4}$

\item Discount ratio $\gamma$: $0.99$

\item GAE coefficient: $0.95$

\item Gradient clip by norm: $0.5$

\item Rollout length: $2048$ environment steps

\item Optimization epochs: $10$

\item Optimization batch size: $64$

\item Action probability ratio clip: $0.2$
\end{itemize}

\textbf{Computing Infrastructure:} We conducted our experiments
on an Intel® Core™ i9-9900X CPU @ 3.50GHz with a single thread
and process with PyTorch.

\end{document}